%% file: ms.tex
\documentclass[twoside,twocolumn]{article}
\usepackage{microtype} 
\usepackage{graphicx}
\usepackage{booktabs} 
\usepackage{tikz}
\usetikzlibrary{calc,positioning}
\usetikzlibrary{shapes.geometric}  
\usepackage{rotating}
\usetikzlibrary{positioning,arrows.meta,calc}
\usepackage{dsfont}
\usepackage{fancyhdr}
\usepackage[T1]{fontenc}
\usepackage{microtype}
\usepackage[english]{babel}
\usepackage[hmarginratio=1:1,top=25mm, left=20mm, columnsep=10pt]{geometry} 
\usepackage[hang, small,labelfont=bf,up,textfont=up,up]{caption} 
\usepackage{abstract}

\usepackage{titling}
\usepackage{url}
\usepackage{mathtools}
\usepackage{amsmath}
\usepackage{amsthm}
\usepackage{amssymb}
\usepackage{natbib}
\usepackage{verbatim}
\usepackage{caption}
\usepackage{enumitem}
\usepackage{wrapfig}\usepackage{adjustbox}
\usepackage{bm}
\usepackage{makecell,multirow,booktabs}
\usepackage{xcolor,colortbl}
\usepackage{hyperref}
\hypersetup{
	colorlinks,
	linkcolor=our-darkgreen,
	urlcolor=our-darkgreen,
	citecolor=our-darkblue}
\usepackage{tabularx}
\usepackage{collcell}
\usepackage{titlesec}
\titlespacing*{\section}
{0pt}{2.5ex plus 1ex minus .2ex}{2.3ex plus .2ex}
\titlespacing*{\subsection}
{0pt}{2.5ex plus 1ex minus .2ex}{2.3ex plus .2ex}
\newtheorem{theorem}{Theorem}

\newtheorem{prop}{Proposition}
\newtheorem{lemma}{Lemma}

\DeclareMathOperator{\Enc}{Enc_{\varphi}}
\DeclareMathOperator{\Dec}{Dec_{\theta}}
\DeclareMathOperator{\diag}{diag}

\DeclareMathOperator{\tr}{tr}

\newcommand{\norm}[1]{\left\lVert#1\right\rVert}
\newcommand{\veq}{\mathrel{\rotatebox{90}{$=$}}}

\definecolor{our-green}{rgb}{0.56, 0.692, 0.195}
\definecolor{our-darkgreen}{rgb}{0.297, 0.348, 0.105}
\definecolor{our-orange}{rgb}{0.881, 0.611, 0.142}
\definecolor{our-red}{rgb}{0.923, 0.386, 0.209}
\definecolor{our-blue}{rgb}{0.368,0.507,0.71}
\definecolor{our-darkblue}{rgb}{0.2,0.3,0.50}
\definecolor{our-violet}{rgb}{0.528,0.471,0.701}
\definecolor{our-brown}{rgb}{0.772,0.432,0.102}
\definecolor{our-lightblue}{rgb}{0.364,0.619,0.782}

\newcommand\setrow[1]{\gdef\rowmac{#1}\ignorespaces}
\newcommand\clearrow{\global\let\rowmac\relax}
\clearrow

\newcolumntype{C}{>{\collectcell\rowmac}c<{\endcollectcell}}
\newcolumntype{R}{>{\collectcell\rowmac}r<{\endcollectcell}}
\newcolumntype{L}{>{\collectcell\rowmac}l<{\endcollectcell}}

\def\eg{e.g.\@}
\def\ie{i.e.\@}

\newcommand{\fig}[1]{Fig.~\ref{#1}}    

\renewcommand{\sec}[1]{Sec.~\ref{#1}} 

\newcommand{\idx}{\ensuremath{^{(i)}}}
\renewcommand{\xi}{\mathbf{x}^{(i)}}
\newcommand{\xip}{\mathbf{x'}^{(i)}}
\newcommand{\zi}{\mathbf{z}^{(i)}}

\newcommand{\mui}{\bm{\mu}^{(i)}}
\newcommand{\sigmai}{\bm{\sigma}^{(i)}}
\newcommand{\sigmais}{{\bm{\sigma}^{(i)}}^2}
\newcommand{\bx}{\mathbf{x}}
\newcommand{\wi}{\mathbf{w}^{(i)}}
\let\tilde\widetilde
\newcommand{\mb}[1]{\mathbf{#1}}
\DeclarePairedDelimiterX{\KL}[2]{D_{KL}\Big(}{\Big)}{%
  #1\;\delimsize\|\;#2%
}
\DeclareMathOperator*{\argmax}{arg\,max}
\DeclareMathOperator*{\argmin}{arg\,min}

\newcommand{\E}{\mathbb{E}}
\newcommand{\R}{\mathbb{R}}

\pretitle{\begin{center}\large\bfseries}
\posttitle{\end{center}}
\title{Demystifying Inductive Biases for $\beta$-VAE Based Architectures} 
\fancyheadoffset[RE,LO]{0cm}
\author{
{Dominik Zietlow, Michal Rol\'{i}nek, Georg Martius} \\[1ex]
\normalsize Max Planck Institute for Intelligent Systems\\
\normalsize Max-Planck-Ring 4, 72076 T\"ubingen, Germany \\
\normalsize {\tt\small \{dzietlow, mrolinek, gmartius\}@tue.mpg.de} 
}
\date{}
\renewcommand{\maketitlehookd}{
\begin{abstract}
\noindent The performance of $\beta$-Variational-Autoencoders ($\beta$-VAEs) and their variants on learning semantically meaningful, disentangled representations is unparalleled. On the other hand, there are theoretical arguments suggesting the impossibility of unsupervised  disentanglement. In this work, we shed light on the inductive bias responsible for the success of VAE-based architectures. We show that in classical datasets the structure of variance, induced by the generating factors, is conveniently aligned with the latent directions fostered by the VAE objective. This builds the pivotal bias on which the disentangling abilities of VAEs rely. By small, elaborate perturbations of existing datasets, we hide the convenient correlation structure that is easily exploited by a variety of architectures. To demonstrate this, we construct modified versions of standard datasets in which (i) the generative factors are perfectly preserved; (ii) each image undergoes a mild transformation causing a small change of variance; (iii) the leading \textbf{VAE-based disentanglement architectures fail to produce disentangled representations whilst the performance of a non-variational method remains unchanged}.
The construction of our modifications is nontrivial and relies on recent progress on mechanistic understanding of $\beta$-VAEs and their connection to PCA.  We strengthen that connection by providing additional insights that are of stand-alone interest.
\end{abstract}
}

\begin{document}
\thispagestyle{empty}
\maketitle
\section{Introduction}
The task of unsupervised learning of \textit{interpretable} data representations has a long history. From classical approaches using linear algebra \eg\ via Principal Component Analysis (PCA) \cite{pca} or statistical methods such as Independent Component Analysis (ICA) \cite{ica} all the way to more recent approaches that rely on deep learning architectures.

The cornerstone architecture is the Variational Autoencoder \cite{KingmaWelling2014:VAE} (VAE) which clearly demonstrates both high semantic quality as well as good performance in terms of \textit{disentanglement}. Until today, derivates of VAEs \cite{higgins2016beta,factor-vae,chen2018isolating,kumar2017variational,klindt2020towards} excel over other architectures in terms of disentanglement metrics. The extent of the VAE's success even prompted recent deeper analyses of its inner workings \cite{rolinek2019variational,understanding-disent,chen2018isolating,2018arXiv181202833M}. 

If we treat the overloaded term disentanglement to the highest of its aspirations, as the ability to recover the \textit{true generating factors} of data, fundamental problems arise. As explained by \citet{locatello2019challenging}, already the concept of generative factors is compromised from a statistical perspective: two (or in fact infinitely many) sets of generative factors can generate statistically indistinguishable datasets.  Yet, the scores on the disentanglement benchmarks are high and continue to rise. This apparent contradiction stems from biases present in used datasets, metrics, and architectures. It was concluded in \citet{JMLR:v21:19-976} that
\begin{quote}
    \textit{[...] future work on disentanglement learning should be explicit about the role of inductive biases and (implicit) supervision [...]}
\end{quote}
which did not happen for the majority of existing approaches. We close this gap for VAE-based architectures on the two most common datasets, namely dSprites \cite{dsprites17} and Shapes3d \cite{3dshapes18}.

The main hypothesis of this work is that all unsupervised, VAE-based disentanglement architectures are successful because they exploit the same structural bias in the data. The ground truth generating factors are well aligned with the nonlinear principal components that VAEs strive for. This bias can be reduced by introducing a \textbf{small change of the local correlation structure} of the input data, which, however, \textbf{perfectly preserves the set of generative factors}. We evaluate a set of approaches on slightly modified versions of the two leading datasets in which each image undergoes a modification inducing little variance. We report drastic drops of disentanglement performance on the altered datasets.

On a technical level, we build on the findings by \citet{rolinek2019variational} who argued that VAEs recover the \textit{nonlinear principal components} of the data. In other words, they recover a set of scalars that embody the sources of variance through a nonlinear mapping, similarly to PCA in the linear setting. We extend their argument by an additional finding that further strengthens this connection. The small modifications of the datasets we propose aim to change the leading principal components by adding modest variance to a set of alternative candidates. The ``to-be'' leading principal components are specific to each dataset, but they are automatically determined in a consistent fashion. 

\section{Related work}
The related work can be categorized into three research questions: i) defining disentanglement and metrics capturing the quality of latent representations; ii) architecture development for unsupervised learning of disentangled representations; and iii) understanding the inner workings of existing architectures, as for example of $\beta$-VAEs. This paper is built upon results from all three lines of work.

\paragraph{Defining disentanglement.}
Defining the term \textit{disentangled representation} is an open question \cite{higgins2018towards}.
The presence of learned representations in machine learning downstream tasks, such as object recognition, natural language processing, and others, created the need to \textit{``disentangle  the  factors  of  variation''}~\cite{bengio2013representation} early on. 
This vague interpretation of disentanglement is inspired by the existence of a low dimensional manifold that captures the variance of higher dimensional data. As such, finding a factorized, statistically independent representation became a core ingredient of disentangled representation learning and dates back to classical ICA models~\cite{ica,bell1995information}.\\
For some tasks, the desired feature of a disentangled representation is that it is \textit{semantically meaningful}. Prominent examples can be found in computer vision \cite{shu2017neural, liao2020towards} and in research addressing the interpretability of machine learning models \cite{adel2018discovering, kim2019interpretable}.\\
Based on group theory and symmetry transformations, \citet{higgins2018towards} provides the \textit{``first principled definition of a disentangled representation''}. Closely related to this concept is also the field of causality in machine learning~\citep{scholkopf2019causality, suter2019robustly}, more specifically the search for causal generative models~\cite{besserve2018group, besserve2020theory}. In terms of implementable metrics, a variety of quantities have been introduced, such as the $\beta$-VAE score \cite{higgins2016beta}, SAP score \cite{kumar2017variational}, DCI scores \cite{eastwood2018framework} and the Mutual Information Gap (MIG, \citet{chen2018isolating}).
\paragraph{Architecture development.} The leading architectures for disentangled representation learning are based on VAEs~\cite{KingmaWelling2014:VAE}. Despite originally developed as a generative modeling architecture, its variants have proven to excel at representation learning tasks. In particular, the $\beta$-VAE performs remarkably well. It exposes the trade-off between reconstruction and regularization via an additional hyperparameter.
Other architectures have been proposed that additionally encourage statistical independence in the latent space, \eg~FactorVAE \citep{kim2018disentangling} and \mbox{$\beta$-TC-VAE} \citep{chen2018isolating}. The DIP-VAE \citep{kumar2017variational} suggests using moment-matching to close the distribution gap introduced in the original VAE paper. Using data with auxiliary labels, \eg~time indices of time series data, for which the conditional prior latent distribution is factorized, allowed \citet{pmlr-v108-khemakhem20a} to circumvent the unidentifiability of previous models. Similarly, \citet{klindt2020towards} used a sparse temporal prior to develop an identifiable model that also performs well on natural data. In this work, we also compare against representations learned by Permutation Contrastive Learning (PCL) \cite{hyvarinen2017nonlinear}. This non-variational method conducts nonlinear ICA also assuming temporal dependencies between the sources of variance. The PCL objective is based on logistic regression.
\paragraph{Understanding inner workings.}
With the rising success and development of VAE based architectures, the question of understanding their inner working principles became dominant in the community. One line of work tries to answer the question why these models disentangle  at all~\citep{understanding-disent}. Another closely related line of work showed the tight connection between the vanilla ($\beta$-)VAE objective and (probabilistic) PCA~\citep{tipping1999probabilistic}~\citep{rolinek2019variational, lucas2019don}. Building on these findings, novel approaches for model selection were proposed~\citep{Duan2020Unsupervised}, emphasizing the value of thoroughly understanding these methods. On a less technical side, \citet{locatello2019challenging} conducted a broad set of experiments, questioning the relevance of the specific model architecture compared to the choice of hyperparameters and the variance over restarts. They also formalized the necessity of inductive biases as a strict requirement for unsupervised learning of disentangled representations. Our experiments are built on their code-base.

\section{Background}
\subsection{Quantifying Disentanglement}

Among the different viewpoints on disentanglement, we follow the recent literature and focus on the connection between the discovered data representation and a set of \textit{generative factors}.

Multiple metrics have been proposed to quantify this connection. 
Most of them are based on the understanding that, ideally, each generative factor is encoded in precisely one latent variable. 
This was captured concisely by \citet{chen2018isolating}, who proposed the Mutual Information Gap (MIG) -- the mean difference (over the $N_w$ generative factors) of the two highest mutual information between a latent coordinate and the single generating factor, normalized by its entropy. For the entropy $H(w_i)$ of a generating factor and the mutual information $I(w_i; z_k)$ between a generating factor and a  latent coordinate, the MIG is defined as
\begin{align}
    \frac{1}{N_w} \sum_{i=1}^{N_w} \frac{1}{H(w_i)} \left( \max_{k}I\left(w_i; z_k \right) - \max_{k \neq k'} I\left(w_i; z_{k} \right) \right),
\end{align}
where $k'=\argmax_{\kappa} I\left(w_i, z_{\kappa}\right)$.
More details about MIG, its implementation, and an extension to discrete variables can be found in \citep{chen2018isolating, rolinek2019variational}. Multiple other metrics were proposed such as SAP score \cite{kumar2017variational}, FactorVAE score \cite{factor-vae} and DCI score \cite{eastwood2018framework} (see the supplementary material of \citet{klindt2020towards} for extensive descriptions).

\subsection{Variational Autoencoders and the Mystery of a Specific Alignment}

Variational autoencoders hide many intricacies and attempting to compress their exposition would not do them justice. For this reason, we limit ourselves to what is crucial for understanding this work: the objective function. For a well-presented description of VAEs, we refer the reader to \cite{doersch2016tutorial}.

As is common in generative models, VAEs aim to maximize the log-likelihood objective 
\begin{align} \label{eq:loglhood}
\sum_{i=1}^N \log p\left(\xi\right),
\end{align}
in which $\{\xi\}_{i=1}^N=\mathcal{X}$ is a dataset consisting of $N$ i.i.d.\ samples $\xi$ of a multivariate random variable $\mb{X}$ that follows the true data distribution. The quantity $p(\xi)$ captures the probability density of generating the training data point $\xi$ under the current parameters of the model. This objective is, however, intractable in its general form. For this reason, \citet{KingmaWelling2014:VAE} follow the standard technique of variational inference and introduce a tractable Evidence Lower Bound (ELBO):
\begin{align}
 \E_{q(\mb{z} \mid \xi)} \log p\left(\xi \mid \mb{z}\right) + \KL{q(\mb{z} \mid \xi)}{p(\mb{z})} \label{eq:raw_elbo}.
\end{align}
Here, $\mb{z}$ are the latent variables used to generate samples from $\mb{X}$ via a parameterized stochastic decoder $p(\xi \mid \mb{z})$.

The fundamental question of \textit{``How do these objectives promote disentanglement?''} was first asked by \citet{understanding-disent}. This is indeed \textit{far from obvious}; in disentanglement the aim is to encode a fixed generative factor in \textit{precisely} one latent variable. From a geometric viewpoint, this requires the latent representation to be \textbf{axis-aligned} (one axis corresponding to one generative factor). 
This question becomes yet more intriguing after noticing (and formally proving) that both objective functions (\ref{eq:loglhood}) and (\ref{eq:raw_elbo}) are \textit{invariant under rotations} for rotationally symmetric latent space priors, as the ubiquitous $p(\mb{z}) = \mathcal{N}(0, \mathds{1})$ \citep{understanding-disent, rolinek2019variational}. 
In other words, any rotation of a fixed latent representation results in the same value of the objective function and yet $\beta$-VAEs consistently produce representations that are axis-aligned and in effect are isolating the generative factor into individual latent variables.
\begin{figure*}
    \centering
    \includegraphics[width=0.95\linewidth]{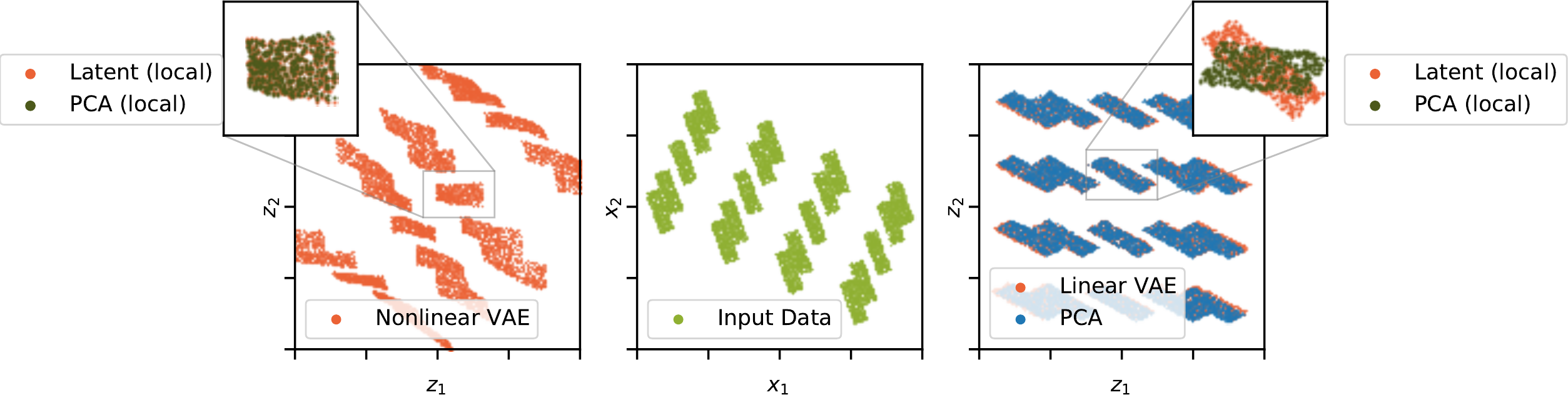}
        \caption{
        Distribution of latent encodings for an input distributed as depicted in the middle (data dimensionality equals latent dimensionality). The linear VAE's encoding matches the PCA encoding remarkably well (right); both focus on aligning with axes based on the global variance. The nonlinear VAE (left) is, however, more sensitive to local variance. It picks up on the natural axis alignment of the microscopic structure. The insets show the enlarged area and PCA performed only on the local subset of the point cloud. Our argument in this work is that misaligning the microscopic structure with respect to the ground truth generating factors leads to decreased \textit{convenient} bias in the data.
    }    
    \label{fig:comparison_pca_vae}
\end{figure*}
\subsection{Resolution via Nonlinear Connections to PCA} \label{sec:pca_vae}

A mechanistic answer to the question raised in the previous subsection was given by \citet{rolinek2019variational}. 
The formal argument showed that under specific conditions which are typical for $\beta$-VAEs (called \textit{polarized regime}), the datapoint-wise linearization of the model  performs PCA in the sense of aligning the ``sources of variance'' with the local axes.
\textbf{The resulting alignment often coincides with finding the components of the datasets ground truth generating factors}. \fig{fig:comparison_pca_vae} illustrates the difference between local and global PCA. Note that the principal directions of a non-degenerate  uniform distribution are the Cartesian axes. PCA as a linear transformation is aligning the embedding following the overall (global) variance. Nonlinear VAEs are aligning the latent space according to the local structure (the local principal components of the almost uniform clusters).
This behavior stems from the convenient but uninformed choice of a \textit{diagonal posterior}, which breaks the symmetry of (\ref{eq:loglhood}) and (\ref{eq:raw_elbo}). This connection with PCA was also reported by \citet{Stuehmer2020:ISAVAE}, alternatively formalized by \citet{lucas2019don} and converted into performance improvements in an unsupervised setting by \citet{Duan2020Unsupervised}. Strictly speaking, the formal statements of \citet{rolinek2019variational} are limited and only claim that $\beta$-VAEs strive for local orthogonality which, in the linear case, is a strong similarity to PCA.

\subsection{Linear vs. Nonlinear Embeddings}
One less obvious observation is that the ``isolation'' of different sources of variance relies on the non-linearity of the decoder. The region in which the linearization of the decoder around a fixed $\mui(\xi)$ is a reasonable approximation suggests a certain radius of the relevant local structure. Since in many datasets the local principal components are well aligned with the intuitively chosen generating factors, $\beta$-VAEs recover sound global principal components. If, however, the local structure obeys a different ``natural'' alignment, the VAE could prefer it, and in return not disentangle the ground truth generating factors.

\section{Methods}
We first tighten the connection between VAEs and PCA, secondly introduce the general data generation scheme of commonly used disentanglement datasets, and lastly turn this understanding into an experimental setup that allows for empirical confirmation that the success of VAE based architectures mostly relies on the local structure of the data.

\subsection{Connection to PCA}
The argument established by \citet{rolinek2019variational} is technically incomplete to justify the equivalence of linear VAEs and PCA. Strictly speaking, the core message of that work is that VAE decoders tend to be locally orthogonal. The actual alignment of the latent space is insufficiently described by that finding. However, \citet{lucas2019don} argue for the similarity of linear VAEs to probabilistic PCA. We now show a more technical connection between classical PCA and linear VAEs which allows for easier understanding of the consequent subsections. We try to stay close to the language of \citet{rolinek2019variational} and partially reuse their arguments.

The canonical implementation of the $\beta$-VAE uses a normal posterior with diagonal covariance matrix and a rotationally symmetric $p(\mb{z}) = \mathcal{N}(0, \mathds{1})$ latent prior. This, together with a Gaussian decoder model, turns the ELBO (\ref{eq:raw_elbo}) into the tractable loss function
\begin{align}
    \mathcal{L} &= \E_i \left( \mathcal{L}\idx_\mathrm{rec} + \beta  \mathcal{L}\idx_\mathrm{KL}\right)
\end{align}
\begin{align*}
    \mathcal{L}_\mathrm{rec} &= \norm{\Dec ( \Enc (\xi)) - \xi }^2  \\\nonumber
    \mathcal{L}_\mathrm{KL} &=  \frac{1}{2} \sum_j \left({\mui}_j^2 + {\sigmai}_j^2 - \log (\sigmais_j) - 1 \right)
    \label{eq:loss_vae}
\end{align*}
for an encoder $\Enc$ parameterized by $\varphi$, a decoder $\Dec$ parameterized by $\theta$, and $\zi = \Enc(\xi) = \mui(\xi) + \bm{\varepsilon}\idx,\ \bm{\varepsilon}\idx \sim \mathcal{N}(0, \sigmais(\xi))$. Since $\zi$ is unbiased around $\mui(\xi)$, we find that
\begin{align}
    \mathcal{L}_\mathrm{rec} =&\ \E_i \left(\mathcal{L}^{\mu}_\mathrm{rec}(\xi) + \mathcal{L}^\mathrm{stoch}_\mathrm{rec}(\xi) \right)\\\nonumber
    \mathcal{L}^\mathrm{stoch}_\mathrm{rec}(\xi) =& \norm{\Dec ( \Enc (\xi)) - \Dec ( \mui ) }^2\\\nonumber
    \mathcal{L}^{\mu}_\mathrm{rec}(\xi) =& \norm{\Dec ( \mui ) - \xi }^2.
\end{align}
\setlength{\belowcaptionskip}{-10pt}
\input{procedure_visualisation}
\setlength{\belowcaptionskip}{0pt}
We assume linear models $\mui = M_E \xi$, $\Dec(\zi) = M_D \zi$ and denote the SVD decomposition of $M_D$ as $M_D = U\Sigma V^\top$.
We can now state a constraint optimization problem of a simplified VAE objective as
\begin{align}
    \min_{\Sigma, U, V} &\E_i \left( \norm{U\Sigma V^\top \bm{\varepsilon}\idx}^2 \right)\label{eq:simplified_optimisation}\\
    \mathrm{s.t.}
    \quad 
    & \E_i \left(\mathcal{L}_\mathrm{\approx KL}\idx\right) = c_{\approx\mathrm{KL}}.\label{const:simplified_optimisation2}
\end{align}
where only the stochastic part of the reconstruction loss is minimized and $c_{\approx\mathrm{KL}}$ is a constant. The term $\mathcal{L}_\mathrm{\approx KL}$ is the KL loss in the polarized regime, where $\sigmais \ll -\log(\sigmais)$ (element-wise):
\begin{align}
    \mathcal{L}_\mathrm{\approx KL} = \sum_j \left({\mui}_j^2 - \log (\sigmais_j)\right).
\end{align}

The 'decoder matrix' of the classical PCA contains the eigenvectors of the covariance matrix $C$. By SVD decomposing the zero-mean data matrix $X = U_X\Sigma_X V_X^\top$, we find
\begin{align}
    C = X^\top X = V_X \Sigma_X^2 V_X^\top.
\end{align}
For encoding data with PCA, the eigenvectors of $V_X$ are typically sorted according to their eigenvalue by a permutation matrix $P$, which leads to the PCA decoder as
\begin{align}
    M_\mathrm{PCA} = V_X^\top\Sigma_X^{2} P.
\end{align}
To tighten the connection between VAEs and PCA, we compare $M_D = U\Sigma V^\top$ to $M_\mathrm{PCA} = V_X^\top\Sigma_X^{2} P$.
\begin{theorem}[Linear VAEs perform PCA]
For any $X \in \R^{n\times m}$, the solution to (\ref{eq:simplified_optimisation}, \ref{const:simplified_optimisation2}) fulfils
\begin{align}
    \Sigma^\star,\ &U^\star, V^\star = \argmin_{\Sigma, U, V} \E_i \left( \norm{U\Sigma V^\top \bm{\varepsilon}\idx}^2 \right),\\\nonumber
    &V^\star\quad \text{is\ a\ signed\ permutation\ matrix},\\\nonumber
    &U^\star = V_X^\top.
\end{align}
\end{theorem}
It was known for long that linear autoencoders, trained on $L^2$ reconstruction loss, span the same space as PCA \cite{bourlard1988auto,baldi1989neural}. The additional similarity that VAEs produce orthogonal mappings, like PCA, was presented by \cite{rolinek2019variational}. With the final connection presented here, even the alignment of the embedding is shown to be identical. For the sake of brevity, the proofs of the statements can be found in the supplementary material. 

Although this does not directly translate to a universal statement about the linearization of a nonlinear model, it provides an intuition for that case as well. An important observation is that \textbf{the alignment of the latent space is mostly driven by the distribution of the latent noise}. When generalizing this statement to the linearization of a nonlinear decoder, the effect of the noise stays local. As a consequence, local changes of the data distribution can potentially lead to a disruptive change in the latent alignments, without inducing large global variance. This idea is depicted in \fig{fig:visualization_modification}.

\subsection{The Generative Process}
The standard datasets for evaluating disentanglement all have an explicit generation procedure. Each data point $\xi \in \mathcal{X}$ is an outcome of a generative process $g$ applied to input $\wi \in \mathcal{W}$. Imagine that $g$ is a function rendering a simple scene from its specification $w$ containing \textit{as its coordinates} the background color, foreground color, object shape, object size, etc. By design, the individual generative factors are statistically independent in $\mathcal{W}$. All in all, the dataset $\mathcal{X} = \left(\bx^{(1)}, \bx^{(2)}, \dots, \bx^{(n)} \right)$ is constructed with $\xi=g(\wi)$, where $g$ is a mapping from the generative factors to the corresponding data points.

In this paper, we design a modification $\Tilde g$ of the generative procedure $g$ that changes the local structure of the dataset $\mathcal{X}$, whilst barely distorts each individual data point. In particular, for each $\xi \in \mathcal{X}$, we have under some distance measure $d(\cdot,\cdot)$, that
\begin{align}
 d\bigl( \xi, \Tilde g(\wi)\bigr) \leq \varepsilon.
    \label{eq:constraint_manipulation}
\end{align}

How to design $\tilde g$ such that despite an $\varepsilon$-small modification, VAE-based architectures will create an entangled representation? 
Following the intuition from \sec{sec:pca_vae}, \fig{fig:comparison_pca_vae} and \fig{fig:visualization_modification}, we \textit{misalign} the local variance with respect to the generating factors in order to promote an alternative (entangled) latent embedding. This is precisely the step from (iii) to (iv) in \fig{fig:visualization_modification}.

To avoid hand-crafting this process, we can exploit the following observation.
VAE-based architectures suffer from large performance variance over \eg\ different random initializations.
This hints at an existing ambiguity: two or more candidates for the latent coordinate system are competing minima of the optimization problem.
Some of these solutions perform well, others are ``bad'' in terms of disentanglement -- they correspond to (ii) and (iii) in \fig{fig:visualization_modification} respectively. 
Below, we elaborate on how to foster the entangling and diminish the disentangling solutions. 

Our modifications are not an implementation of \citep[Theorem~1]{locatello2019challenging}. We \textbf{do not modify the set of generative factors, but slightly alter the generating process} to target a specific subtlety in the inner working of VAEs.

Given any dataset, our modification process has three steps:
\begin{enumerate}
    \item[(i)] Find the most disentangled and the most entangled latent space alignment that a $\beta$-VAE produces over multiple restarts.
    \item[(ii)] Optimize a generator that manipulates images to foster and diminish their suitability for the entangled and disentangled model respectively.
    \item[(iii)] Apply the manipulation to the whole dataset and compare the performance of models trained on the original and the modified dataset.

\end{enumerate}
\setlength{\belowcaptionskip}{-10pt}
\input{architecture_figure}
\setlength{\belowcaptionskip}{0pt}
\subsection{Choice of Fostered Latent Coordinate System}
\label{sec:scalar}

Over multiple restarts of $\beta$-VAE, we pick the model with the lowest MIG score. This gives us an entangled alignment that is expressible by the architecture. Although any choice of metric is valid for this model selection (e.g. UDR \cite{Duan2020Unsupervised}), we chose MIG for the sake of simplicity. The latent variables of each of the models capture the nonlinear principal components of the data. Similarly to PCA, we can order them according to the variance they induce. The order is inversely reflected by the magnitude of the latent noise values. We find the $j\text{'th}$ principal components $s^{(i)}_j$ as
\begin{align}
    s^{(i)}_j\big(\xi\big) = \mathrm{enc}\big( \xi\big)_{k^{(j)}}\\ \quad k^{(j)} = \argmin_{l \not\in \{k^{(0)}, k^{(1)}, \dots, k^{(j-1)}\}} \big\langle \bm{\sigma}_l^2 \big\rangle.\label{eqn:s}
\end{align}
This procedure of sorting the most \textit{important} latent coordinates is consistent with \cite{higgins2016beta} and \cite{rolinek2019variational}.  The analogy to PCA is that the mapping $s^{(j)}(\xi)$ gives the $j\text{'th}$ coordinate of $\xi$ in the new (nonlinear) coordinate system.

\subsection{Dataset Manipulations}
\label{sec:pattern}

We will now describe the modification procedure assuming the data points are $r \times r$ images.
The manipulated data-point $\xip$ 
is of the form $\xip = \xi + \varepsilon m\bigl( \wi \bigr)$
where the mapping $f\colon\R \to \R^r \times \R^r$ is constrained by $\|m(\wi)\|_\infty \le 1$ for every $\wi$. 
Then inequality (\ref{eq:constraint_manipulation}) is naturally satisfied for the maximum norm.

The abstract idea of how to achieve a change of the latent embedding coordinate systems can be visualized using the intuition following from Eq.~(\ref{eqn:s}). We can think of two VAE latent spaces where one is considered disentangled ($\{\mui_\mathrm{dis},  \sigmai_\mathrm{dis}\}$) and the other is entangled ($\{\mui_\mathrm{ent},  \sigmai_\mathrm{ent}\}$), as two sets of nonlinear principal directions, and the variance each of the dimensions capture is reflected in the magnitude of $\sigmai$. We are aiming to alter the dataset such that its entangled representation is superior over the disentangled representation, in the sense of being \textit{cheaper} to decode with respect to the reconstruction loss. In other words, projecting the dataset to the manifold supported by $\zi_\mathrm{ent}$ should result in a lower loss in Eq.~(\ref{eq:loss_vae}) than projecting it to the manifold supported by $\zi_\mathrm{dis}$. A naive way of doing so is by moving each image closer to its projections on the first principal components of the entangled representation and further away from those of the disentangled representation. Instead of hand-crafting this operation, we can optimize for it directly.

This idea can be turned into an end-to-end trainable architecture as depicted in Fig. \ref{fig:modification_architecture}. We want to change the dataset such that it is more convenient to encode it in an entangled way. Starting with two pretrained models, we fix their encoders and keep feeding them the original images. This ensures that the latent encoding stay unchanged, as we want to compare their suitability for reconstruction. The decoders are trained to minimize the reconstruction loss given the entangled representation:
\begin{align*}
    \theta^\star_\mathrm{ent} 
    =& \argmin_{\theta_\mathrm{ent}} \mathcal{L}_{\mathrm{rec}}^\mathrm{ent} \left( \xip, \zi \right),\\
    \theta^\star_\mathrm{dis} = &
    \argmin_{\theta_\mathrm{dis}} \mathcal{L}_{\mathrm{rec}}^\mathrm{dis} \left( \xip, \zi \right).
\end{align*}

We initialize this network with the parameters of the disentangled model $\theta_\mathrm{dis}, \varphi_\mathrm{dis}$ and the entangled model $\theta_\mathrm{ent}, \varphi_\mathrm{ent}$ respectively. We introduce a network to learn the additive manipulation, $m_\psi$. It is trained to minimize the reconstruction loss of the entangled VAE and to increase the loss of the disentangled VAE via its effect on the dataset:
\begin{align*}
    \psi^\star = \argmin_\psi \left( \mathcal{L}_{\mathrm{rec}}^\mathrm{ent} \left( \xip, \zi \right) - \mathcal{L}_{\mathrm{rec}}^\mathrm{dis} \left( \xip, \zi \right) \right).
\end{align*}
It is worth noting that both latent spaces were suitable for reconstructing the images of the original dataset. \textbf{The major play that the network $m_\psi$ has, is to utilize the different ways the noise was distributed across the latent space.}

\section{Experiments}
In order to experimentally validate the soundness of the manipulations, we need to demonstrate the following:

\begin{enumerate}[topsep=0em,itemsep=0em]
    \item \textbf{Effectiveness of manipulations.} Disentanglement metrics should drop on the altered datasets across VAE-based architectures. We do not expect to see changes on non variational methods.
    \item \textbf{Comparison to a trivial modification.} Instead of the proposed method, we modify with uniform noise of the same magnitude. The disentanglement scores for the algorithms on the resulting datasets should not drop significantly, as this change does not alleviate the existing bias.
    \item \textbf{Robustness.} The new datasets should be hard to disentangle even after retuning hyperparameters of the original architectures.
\end{enumerate}
\input{ds_figure}
\subsection{Effectiveness of Manipulations}
We deploy the suggested training for the manipulations on two datasets: Shapes3D and dSprites, leading to manipulations as depicted in \fig{fig:example_manipulations}. 
{
\setlength{\belowcaptionskip}{-0pt}
\begin{table*}[tb] 
\centering
\footnotesize
\caption{MIG Scores for unmodified, modified and noisy datasets. We report the mean and standard deviation over 10 distinct random seeds for each setting. The regular autoencoder serves as a baseline (random alignment). PCL is the only disentangling non-variational model. The modification leads to a significant drop in all variational methods.}\vspace{-.5em}
\begin{adjustbox}{max width=0.75\linewidth}
\begin{tabular}{R@{\hskip1em}*{3}{C}|*{3}{C}<{\clearrow}}
	\toprule
	\setrow{\textbf}
	~ & \multicolumn{3}{C}{dSprites} & \multicolumn{3}{C}{Shapes3d}\\
	~ & orig. & mod. & noise & orig. & mod. & noise
		\\ \midrule
	{\bf AE\quad}
	& $0.09 \pm 0.06$ & -- & -- & $0.06 \pm 0.03$  & --  & --
		\\ \cmidrule{1-7}
	{\bf $\beta$-VAE\quad}
	& $0.23 \pm 0.08$ & $0.07\pm0.09$  & $0.14 \pm 0.07$  & $0.60 \pm 0.31$  & $0.09 \pm 0.14$  & $0.66 \pm 0.05$ 
		\\ \cmidrule{1-7}
	{\bf Fac. VAE\quad}
	& $0.27 \pm 0.11$ & $0.20 \pm 0.12$  & $0.16 \pm 0.08$  & $0.27 \pm 0.18$  & $0.07 \pm 0.05$  & $0.33 \pm 0.20$ 
		\\ \cmidrule{1-7}
	{\bf TC-$\beta$-VAE\quad}
	& $0.25 \pm 0.08$ & $0.14\pm0.10$  & $0.20 \pm 0.04$  & $0.58 \pm 0.20$  & $0.24 \pm 0.16$  & $0.60 \pm 0.11$ 
		\\ \cmidrule{1-7}
	{\bf Slow-VAE\quad}
	& $0.39 \pm 0.08$ & $0.27 \pm 0.08$  & $0.37 \pm 0.09$  & $0.53 \pm 0.19$  & $0.13 \pm 0.08$  & $0.60 \pm 0.10$  
	\\\cmidrule{1-7}
	{\bf PCL\quad}
	& $0.21 \pm 0.03$ & $0.24\pm 0.07$  & $0.24 \pm 0.07$ & $0.44 \pm 0.06$  & $0.47 \pm 0.08$  & $0.40 \pm 0.07$
	\\
		
	\bottomrule
\end{tabular}
\end{adjustbox}
\label{tbl:overview}
\end{table*}
}
In terms of models, we evaluated four VAE-based architectures \cite{higgins2016beta,factor-vae,chen2018isolating,klindt2020towards}, a regular autoencoder \cite{hinton2006reducing}, and, as a non-variational method, PCL \cite{hyvarinen2017nonlinear}, on both the original and manipulated datasets. 
We used the regularization strength reported in the literature (or better tuned values), and took the other hyperparameters from the disentanglement library \cite{locatello2019challenging}. 
For the sake of simplicity and clarity, we restricted the latent space dimension to be equal to the number of ground truth generative factors. 
Most of the architectures have been shown to be capable of pruning the latent space as a consequence of their intrinsic regularization \cite{Stuehmer2020:ISAVAE}. 
Whilst being a perk in real world application scenarios, this behaviour can lead to over- or under-pruning and thereby cloak the actual difference in the alignment of the latent space.
The resulting MIG scores are listed in Tab.~\ref{tbl:overview}, other metrics are listed in the supplementary materials.
Over all variational models, the disentanglement quality is significantly reduced. Interestingly even for SlowVAE, an architecture that supposedly circumvents the non-identifiability problem by deploying a sparse temporal prior, the disentanglement reduces. This indicates that the architecture still builds upon the local data structure more than the temporal sparsity. \textbf{PCL, as a non variational method, performs equivalently well on the original and the modified architecture}, which is a strong indicator that due to the constraint (\ref{eq:constraint_manipulation}), the main sources of global variance remain unaltered. The modifications indeed only attack the subtle bias VAEs exploit.

\subsection{Noisy Datasets}
We replace our modification by contaminating each image with uniform pixel-wise noise $[-\varepsilon, \varepsilon]$. The value of $\varepsilon$ is fixed to the level of the presented manipulations ($0.1$ for dSprites and $0.175$ for Shapes3D). The results are also listed in Tab. \ref{tbl:overview}. The lack of structure in the contamination does not affect the performance in a guided way and leads to very little effect on Shapes3D. The impact on dSprites is, however, noticeable. Due to the comparatively small variance among dSprites images, the noise conceals the variance from the less important generating factors (such as \eg\ orientation).

\begin{figure}[h!]
\begin{center}
    \begin{adjustbox}{max width=\linewidth}
    \includegraphics[width=1\linewidth]{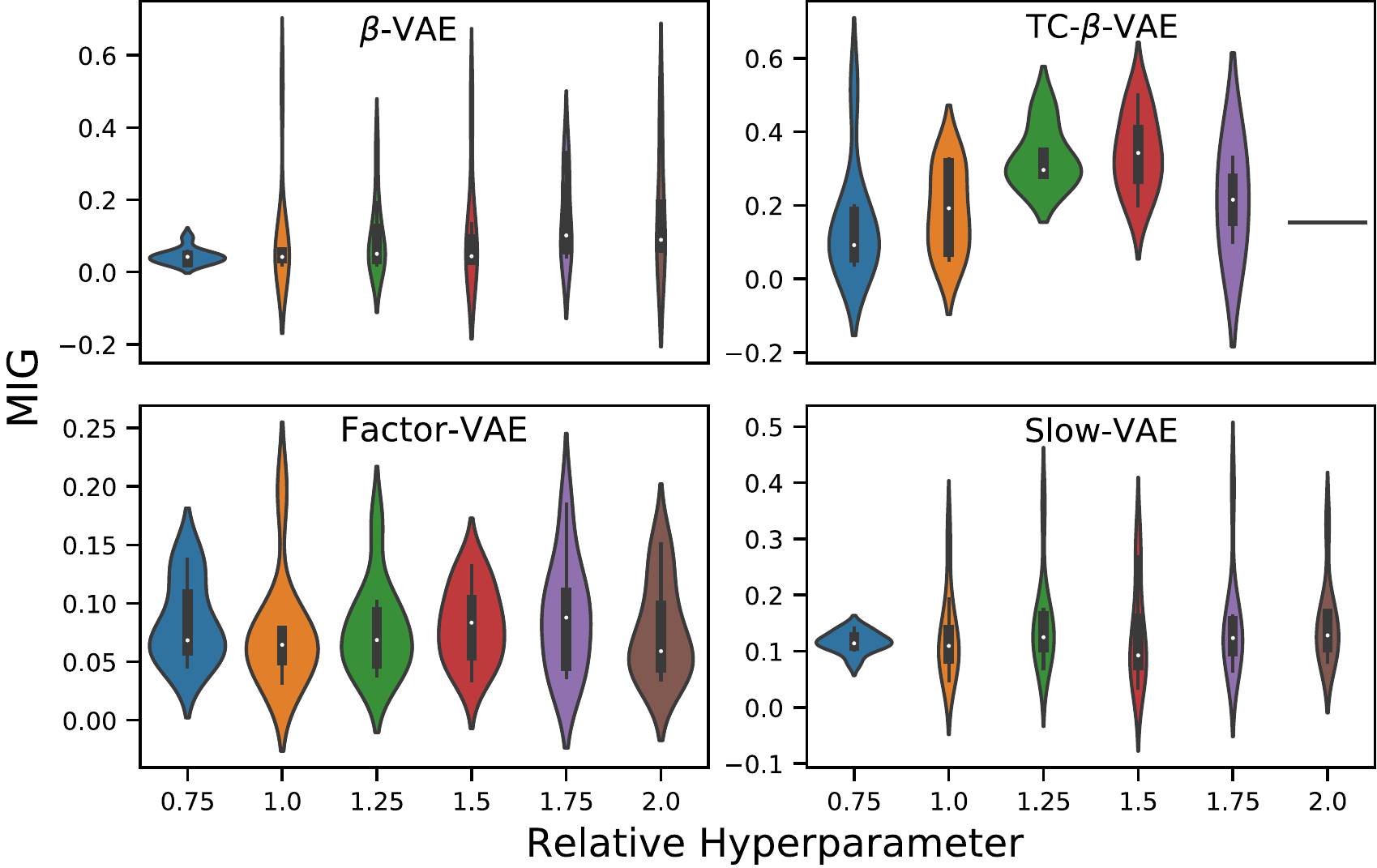}
    \end{adjustbox}
    \caption{MIG scores for scaled literature hyperparameters over $10$ restarts for Shapes3D. Overpruning runs with fewer active units than generating factors were discarded}
    \label{fig:gridsearch}
\end{center}
\end{figure}
\subsection{Robustness over Hyperparameters}
We run a line search over the primary hyperparameter for each architecture. The results are illustrated in \fig{fig:gridsearch}. Overall our modifications seem mostly robust for adjusted hyperparameters. Significant increase in the regularization strength allowed for some recovery. More thorough analysis revealed that this effect starts only once the models reach a level of over-pruning, which is a behavior well known to practitioners. We discard the runs that over pruned the latent space (number of active coordinates, \ie\ $\E\left(\bm{\sigma}_i^2\right) < 0.8$, sinks below the dimensionality of the ground truth generating factors). This effect goes along with decreased reconstruction quality and intrinsically prevents the models from recovering all true generating factors and as such renders these cases uninteresting.

\section{Conclusion}
We have shown that the success of $\beta$-VAE based architectures is mostly based on the structured nature of the datasets they are being evaluated on. Small perturbations of the dataset can alleviate this structure and decrease the bias that such architectures exploit. Interestingly even architectures that are proven to be identifiable, like the Slow-VAE, still owe their success to the same bias. PCL however, as a non-variational method, was unaffected by the small perturbation.

It remains an open question whether the same local structure can reliably be found in real world data on which such architectures could be deployed. If so, fostering the sensitivity of future architectures towards the natural alignment of data could result in a transparent advance of unsupervised representation learning. It would be interesting to investigate and compare the different nonlinear embeddings VAE based architectures find. There are hints of clearly distinct local minima of the optimization problem; their suitability for downstream applications remains unexplored.
\section*{Acknowledgements}
We thank Maximilian Seitzer and Lukas Schott for the fruitful and invaluable discussions. Also, we thank the International Max Planck Research School for Intelligent Systems (IMPRS-IS) for supporting DZ.
\setlength{\bibsep}{2pt plus 0.3ex}
\bibliography{bibliography}
\bibliographystyle{apalike}
\newpage
\appendix

\twocolumn[
\begin{center}
\Large Supplementary Material \\
Demystifying Inductive Biases for (Beta-)VAE Based Architectures
\end{center}
\vspace{-0.5cm}
\input{proof_vis}
\vspace{1cm}
]
\thispagestyle{empty}

\section{Proofs}
\subsection{The Formal Setting}
\label{sec:recap}
The simplified objective stated in this paper as
\begin{align}
    \min_{\Sigma, U, V} &\E_i \left( \norm{U\Sigma V^\top \bm{\varepsilon}\idx}^2 \right)\label{eq:objective_init}\\
    \mathrm{s.t.}
    \quad 
    & \E_i \left(\mathcal{L}_\mathrm{\approx KL}\idx\right) = c_{\approx\mathrm{KL}}.\label{eq:constraint_init}
\end{align}
resembles the minimization problem (20) and (21) from \citet{rolinek2019variational}. They only optimize for distributing the latent noise $\sigmai$ and the orthogonal matrix $V$ of the SVD decomposition of the whole linear decoder and conclude that for $M=U\Sigma V^\top$
\begin{quote}
     \textit{In every global minimum, the columns of $M$ are orthogonal.}
\end{quote}
Which is equivalent to $V$ being a signed permutation matrix (Proposition 1 of \cite{rolinek2019variational}). Without loss of generality, we assume $V = \mathcal{I}$ and rearrange the elements of $\Sigma$ in ascending order and those of $\varepsilon\idx$ in descending order with respect to $\sigmais$.

In the setting of Theorem (1), we consider the mean latent representation $Z$ to be constrained only by the condition $\mathrm{diag}\left(Z^\top Z\right) = \bm{1}$, which reads as ``each active latent variable has unit variance''. Even though, this statement is unsurprising in the context of VAEs, we offer a quick proof of how this follows directly from the KL loss in Lemma \ref{lemma:normalized_latent}. Additionally, we fully fix the matrix $\hat{X}$, which contains the reconstruction of all data-points. The remaining freedom in $U$ and $\Sigma$ has the following nature: for each fixed $U^\top$ (which rotates $\hat{X}$), the nonzero singular values of $\Sigma$ (scaling factors along individual axes in the latent space) are fully determined by the $\mathrm{diag}\left(Z^\top Z\right) = \bm{1}$ requirement. We minimize objective (\ref{eq:objective_init}) under these constraints.
\vspace{-10pt}
\paragraph{Remark} Notice that fixing the reconstructed data-points ensures that the observed effect is entirely independent of the deterministic loss. The deterministic loss, is known to have some PCA-like effects, as it is basically a MSE loss of a deterministic autoencoder. The additional (and in fact stronger) effects of the stochastic loss are precisely the novelty of the following theoretical derivations.

For technical reasons regarding the uniqueness of SVD, we additionally inherit the assumption of \cite{rolinek2019variational} that the random variables $\bm{\varepsilon}\idx$ have distinct variances.

Finally, the orthnormal matrix $U$ acts isometrically and can be removed from the objective (\ref{eq:objective_init}), even though it still plays a vital role in how the problem is constrained. The reduced objective is further conveniently rewritten as a trace as:
\begin{align}
    \min_{\Sigma} \E_i \norm{\Sigma \bm{\varepsilon}\idx}^2 = \min_{\Sigma} \E_i \tr \left( E \Sigma^\top \Sigma E \right) \label{eq:objective},
\end{align}
where $E$ is the diagonal matrix induced by the vector $\bm{\varepsilon}$.

A visualization of the role of $U$, $\Sigma$ and $V$ in the decoding process is illustrated in \fig{fig:svd_proof}.

\subsection{Proof of Theorem 1}
We rewrite the objective in order to introduce $U$, $\hat{X}$, and $Z$ and make use of the constraints $\diag (Z^\top Z) = \bm{1}$ and $\hat{X} = Z \Sigma U$. We have
\begin{align}
    E \Sigma^\top \Sigma E = E\Sigma^\top (Z^\top Z + M)\Sigma E, \label{eq:sigma_with_z}
\end{align}
where $M = \mathcal{I} - Z^\top Z$ is a matrix with $\mathrm{diag}(M) = 0$.
Also, we can expand
\begin{align}
    \Sigma^\top Z^\top Z\Sigma =  U \left( U^\top \Sigma^\top Z^\top\right) \left( Z \Sigma U \right) U ^\top
    = U \hat{X}^\top \hat{X} U^\top \label{eq:sigma_expand}
\end{align}
By combining (\ref{eq:sigma_with_z}) and (\ref{eq:sigma_expand}), we learn that
\begin{align}
   E \Sigma^\top \Sigma E - E U \hat{X}^\top \hat{X} U^\top E = E \Sigma^\top M \Sigma E. \label{eq:whatever}
\end{align}
By repeating Lemma \ref{lemma:vanishing_trace}, we learn that $\mathrm{diag}(E\Sigma^\top M \Sigma E) = 0$, which allows us to use Lemma \ref{lemma:vanishing_trace} yet again, this time on the left-hand side of (\ref{eq:whatever}) and obtain a key intermediate conclusion:
\begin{align}
    \tr \left( E \Sigma^\top \Sigma E \right) =& \tr \left(E U \hat{X}^\top \hat{X} U^\top E \right)
\end{align}

This has a lower bound according to a classical trace inequality (see Proposition \ref{prop:amgm}), as $E U \hat{X}^\top \hat{X} U^\top E$ is positive semi-definite.
\begin{align}
    \tr \left(E U \hat{X}^\top \hat{X} U^\top E\right) \geq\ & n \det \left( E U \hat{X}^\top \hat{X} U^\top E \right)^{1/n}\\
    =\ & n \det \left(E \hat{X}^\top \hat{X} E \right)^{1/n}
\end{align}
with equality if and only if 
\begin{align}
    E U \hat{X}^\top \hat{X} U^\top E = \lambda \mathcal{I}.
\end{align}
For the SVD decomposition $\hat{X} = U_X \Sigma_X V_X^\top$, we see that $\hat{X}^\top \hat{X} = V_X \Sigma_X^2 V_X^\top$ and with $U' = UV_X$ we arrive at
\begin{align}
    U' \Sigma_X^2 U'^\top = \lambda E^{-2}.
\end{align}
The left-hand side gives an SVD decomposition of the diagonal matrix $E^{-2}$. The SVD decomposition of a diagonal matrix is unique up to a signed permutation matrix. The conclusion of Theorem 1 now follows.

\subsection{Auxiliary Statements}
In the following lemma, the vectors $\bm{x}$ and $\bm{y}$ correspond to the mean latent $\bm{\mu}$ and the noise standard deviation $\bm{\sigma}$ respectively. We allow for scaling the latent space and find that the KL loss is minimal for unit standard deviation of the means.
\begin{lemma}[]
\label{lemma:normalized_latent}
For vectors $\bm{x} = (x_{0}, \dots, x_{n}) \in \mathbb{R}^n$, $\bm{y} = (y_{0}, \dots, y_{n}) \in \R^n$ and
$$c = \argmin_{c \in \R} \sum_i \left( c^2{x_i}^2 - \log \left( c^2{y_i}^2\right)\right),$$ it holds that
\begin{align}
    c = \sqrt{\sum_i \left( {x_i}^2 \right)}
\end{align}
\end{lemma}
\begin{proof}
It is easy to inspect that the minimum of $\sum_i \left( c^2{x_i}^2 - \log \left( c^2{y_i}^2\right)\right)$ with respect to $c$ fulfils the statement.
\end{proof}
\begin{prop}[Trace Inequality]
\label{prop:amgm}
For a positive semi-definite $M \in \R^{n\times n}$, that is $M \succcurlyeq 0$, it holds that
\begin{align}
\tr(M) \geq n\det(M)^{1/n}   
\end{align}
with equality if and only if $M = \lambda \cdot  \mathcal{I}$ for some $\lambda \geq 0$.
\end{prop}
\begin{proof}
Let $\lambda_1$, \dots, $\lambda_n$ denote the eigenvalues of $M$, then $\tr (M) = \sum_i \lambda_i$ and $\det (M) = \prod_i \lambda_i$. Since $M \succcurlyeq 0$, we have $\lambda_i \geq 0$ for every $i = 1, \dots, n$. 
Then, due to the classical AM-GM inequality, we have
\begin{align}
\tr(M) = \sum_i \lambda_i \geq n \cdot \left( \prod_i \lambda_i \right)^{1/n} = n\det(M)^{1/n},  
\end{align}
with equality precisely if all eigenvalues are equal to the same value $\lambda \geq 0$. Then by the definition of eigenvalues, the $M - \lambda \mathcal{I}$ has zero rank, and equals to zero as required.
\end{proof}

\begin{lemma}[``Empty diagonal absorbs'']
Let $D\in \R^{m\times m}$ be a diagonal matrix and let $M \in \R^{m\times m}$ be a matrix with zero elements on the diagonal, that is $\mathrm{diag} (M) = 0$. Then $\mathrm{diag}(MD) = \mathrm{diag}(DM) = 0$ and consequently also $\tr(MD) = \tr(DM) = 0$.
\label{lemma:vanishing_trace}
\end{lemma}
\begin{proof}
Follows immediately from the definition of matrix multiplication.
\end{proof}

\section{Experimental Details}
\subsection{Architecture for Manipulations}
The model implemented for $m(\bm{w})$ has almost the same architecture as the CNN decoder as it is implemented in the Disentanglement Library \cite{locatello2019challenging}. The only differences lies in the input MLP which was extended by a single neuron hidden layer. This enforces a compression of the generating factors $\wi$ to some scalar value based on which the modifications are rendered. Both $m$ and the decoders were trained with Adam ($\beta_1=0.9$, $\beta_2=0.999$, $\epsilon=10^{-7}$) and $10^{-4}$ learning rate. To ensure training stability, we train the decoders on three times more batches as the manipulation network and reconstruct five latent samples per image to get a better estimate of the stochastic losses. We achieved a better result on Shapes3D when using an ensemble of four disentangling and four entangling encoder-decoder pairs instead of single models. In order to stay in the same value range as the original images, we ensured normalization of the manipulated images $\xip = \xi + m(\wi)$ by $\xip_\mathrm{norm} = \xi - 2\mathrm{ReLu}(\xi - 1) + 2\mathrm{ReLu}(-\xi)$.

\section{Additional Experiments}
\subsection{Evaluation on Different Metrics}
We have evaluated all architectures on three additional metrics. See Tables (\ref{tbl:dci}, \ref{tbl:factorscore}, \ref{tbl:sapscore}) for the resulting DCI-, FactorVAE- and SAP-Scores.
Figures (\ref{fig:linesearch_dci}, \ref{fig:linesearch_factorscore}, \ref{fig:linesearch_sapscore}) show the scores for a line search of the primary hyperparameter of each architecture. The hyperparameters are listed in Table \ref{tbl:hp}. We used the implementations of the Disentanglement Library.
\setlength{\belowcaptionskip}{-5pt}
\begin{table}[]
    \begin{center}
        \begin{adjustbox}{max width=0.9\linewidth}
            \begin{tabular}{c|c|c}
        \toprule
         \textbf{Architecture} & \textbf{dSprites} & \textbf{Shapes3D} \\\hline
         $\beta$-VAE ($\beta$) & $8$ & $32$\\
         TC-$\beta$-VAE ($\beta$) & $6$ & $32$ \\
         Factor-VAE ($\gamma$) & $35$ & $7$\\
         Slow-VAE ($\beta$) &  $1$ & $1$\\
         \bottomrule
    \end{tabular}
        \end{adjustbox}
            \caption{Primary hyperparameters, for other parameters we used the defaults in the Disentanglement Library or literature values.}
    \label{tbl:hp}
    \end{center}
\end{table}

\subsection{Inspection of Entangled and Disentangled Latent Embeddings}
Over multiple restarts of $\beta$-VAE trainings on the unmodified dataset, we used the runs that achieved highest and lowest MIG scores. Exemplary, \fig{fig:cartesian} and \fig{fig:polar} show two dimensional latent traversals of four disentangled and four entangled $\beta$-VAE representation respectively. The dimension of the latent traversal were hand-picked to encode for the wall hue and the orientation. Interestingly, the disentangled models reliably encode the color in the same way (e.g. starting from green to cyan). The entangled models reliably mix the two generating factors in a very similar way: The color is encoded as the angular component of the two latent dimensions and the orientation as the radial component.\\
\begin{figure}[h!]
\begin{center}
    \begin{adjustbox}{max width=\linewidth}
    \includegraphics[width=\linewidth]{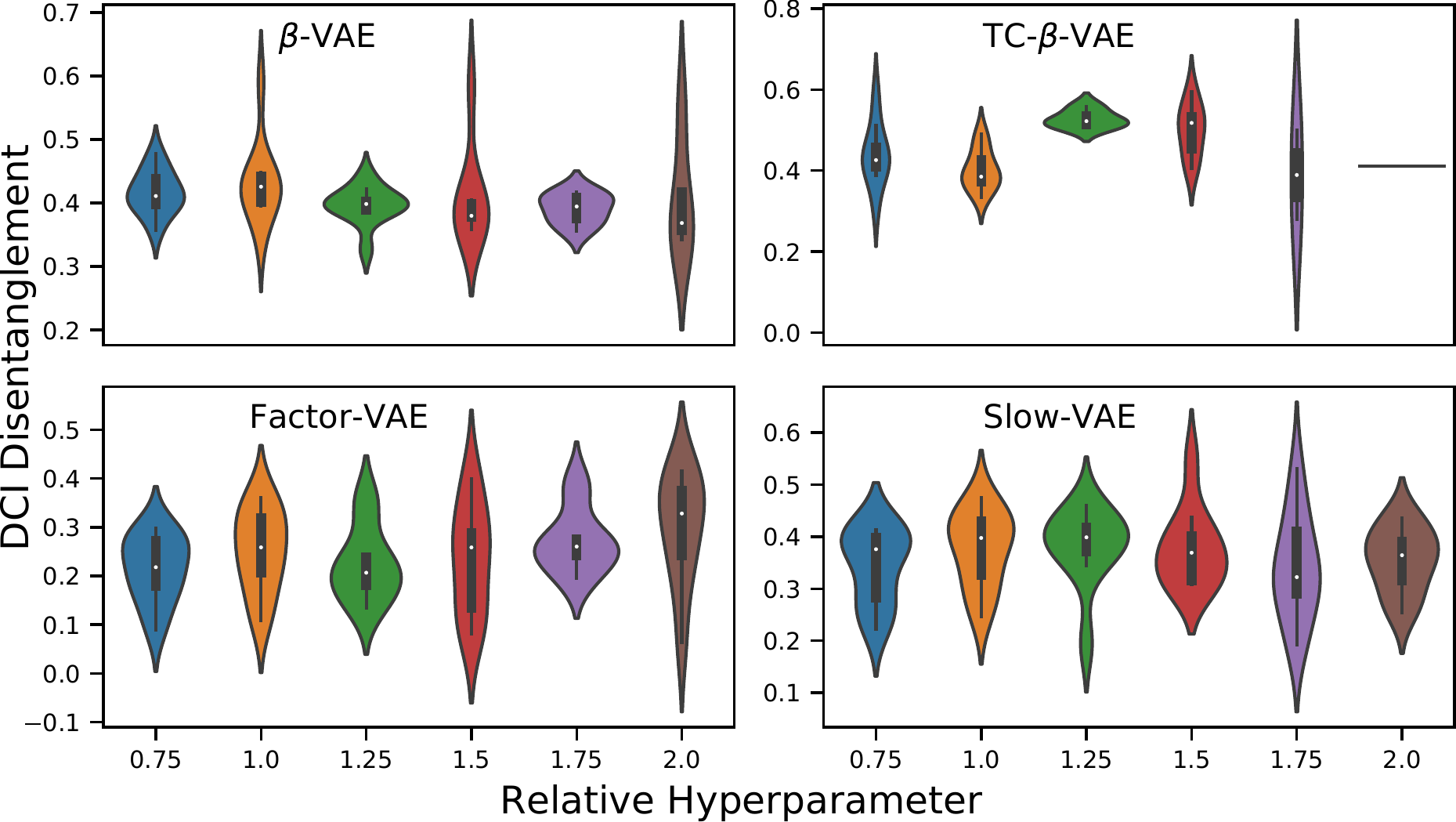}
    \end{adjustbox}
    \caption{DCI scores for scaled literature hyperparameters over $10$ restarts for Shapes3D. Overpruning runs with fewer active units than generating factors were discarded}
    \label{fig:linesearch_dci}
\end{center}
\end{figure}
\onecolumn 
\begin{figure}[h!]
\begin{minipage}[t]{0.48\textwidth}
\begin{center}
    \begin{adjustbox}{max width=\linewidth}
    \includegraphics[width=\linewidth]{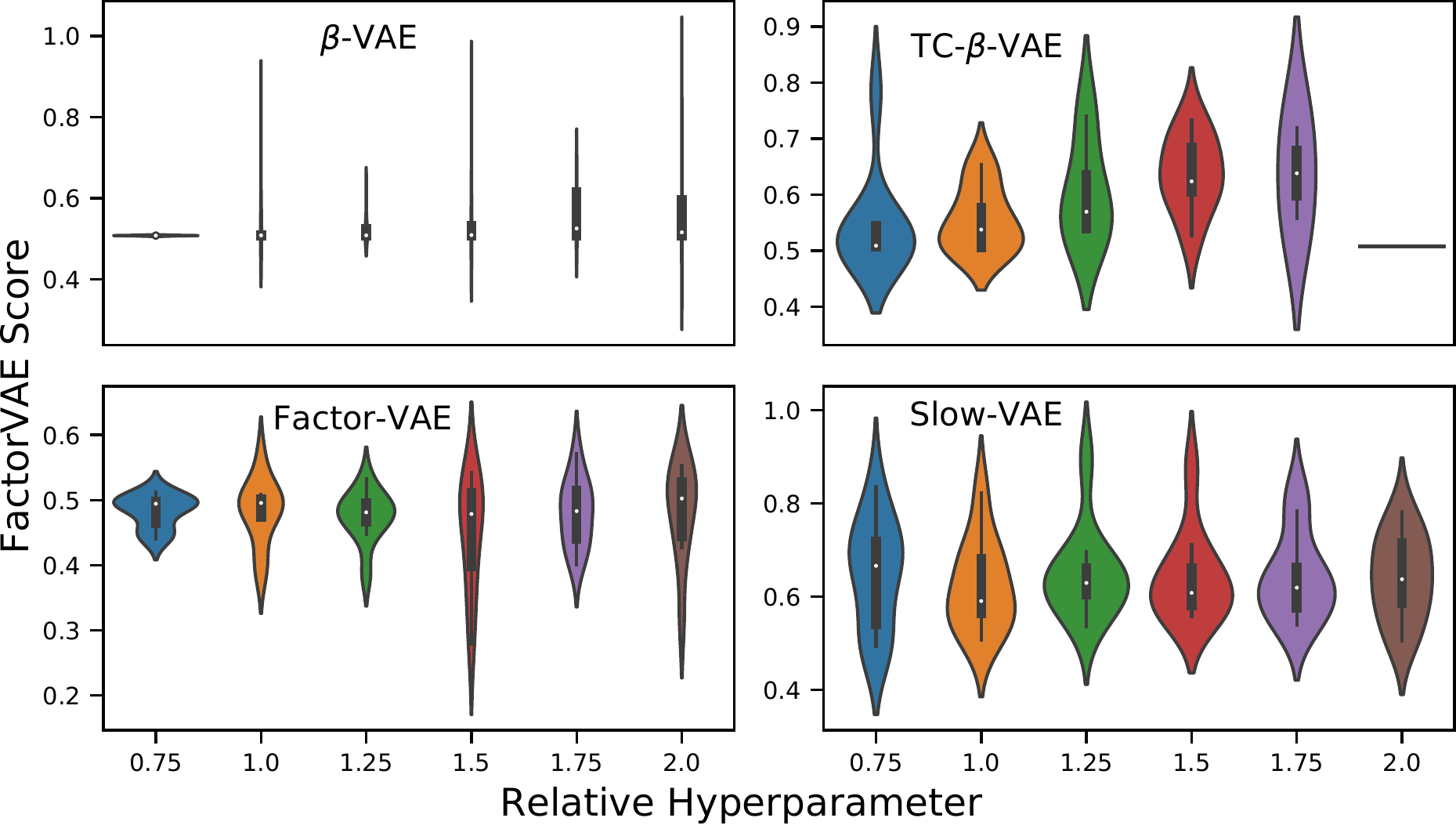}
    \end{adjustbox}
    \caption{FactorVAE scores for scaled literature hyperparameters over $10$ restarts for Shapes3D. Overpruning runs with fewer active units than generating factors were discarded}
    \label{fig:linesearch_factorscore}
\end{center}
\end{minipage}
\hfill
\begin{minipage}[t]{0.48\textwidth}
\begin{center}
    \begin{adjustbox}{max width=\linewidth}
    \includegraphics[width=\linewidth]{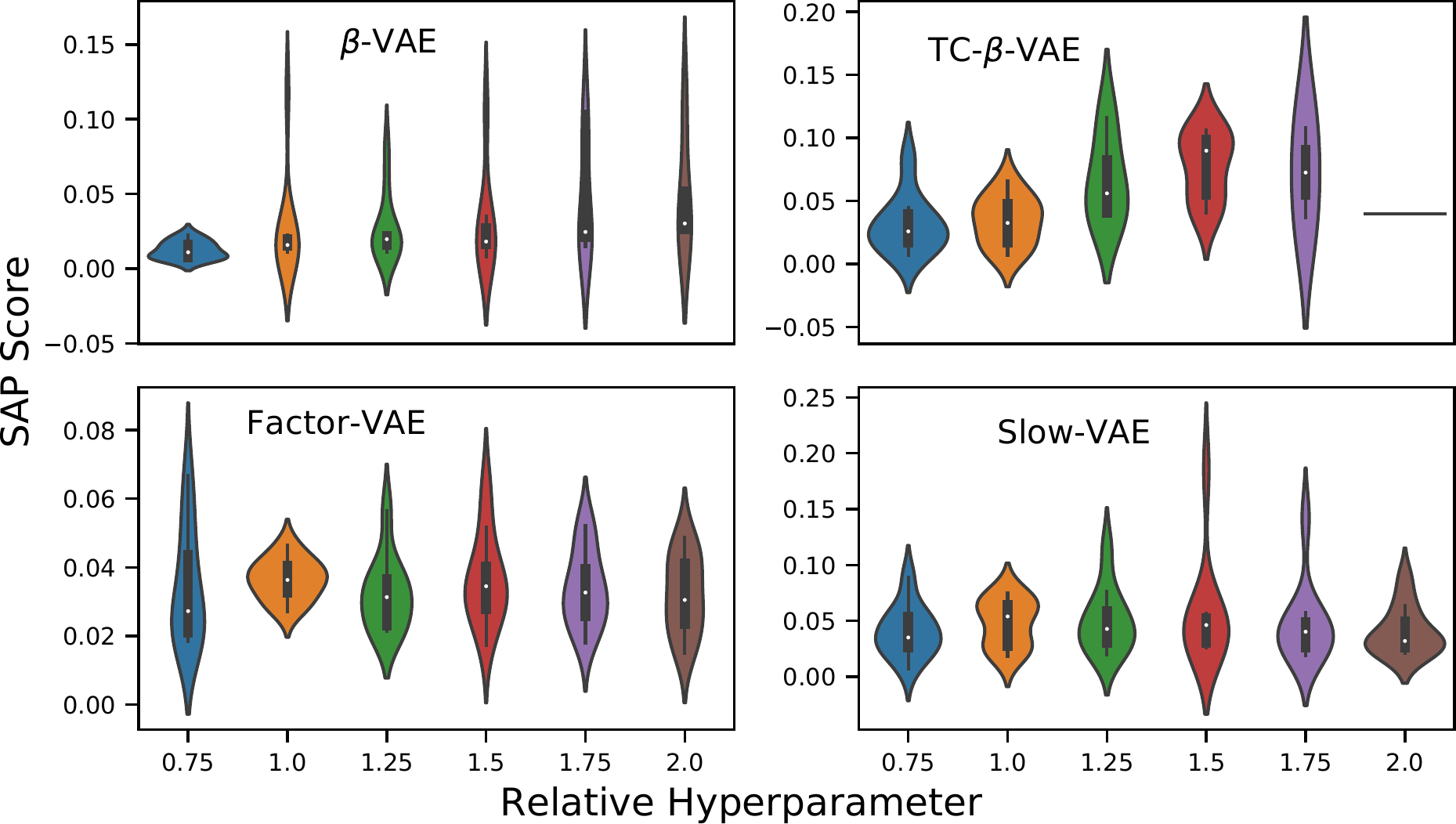}
    \end{adjustbox}
    \caption{SAP scores for scaled literature hyperparameters over $10$ restarts for Shapes3D. Overpruning runs with fewer active units than generating factors were discarded}
    \label{fig:linesearch_sapscore}
\end{center}
\end{minipage}
\end{figure}
\begin{table*}[h!] 
\centering
\footnotesize
\caption{DCI Disentanglement Scores for unmodified, modified and noisy datasets. We report the mean and standard deviation over 10 distinct random seeds for each setting. PCL is the only disentangling non-variational model.}\vspace{-.5em}
\begin{adjustbox}{max width=0.65\linewidth}
\begin{tabular}{R@{\hskip1em}*{3}{C}|*{3}{C}<{\clearrow}}
	\toprule
	\setrow{\textbf}
	~ & \multicolumn{3}{C}{dSprites} & \multicolumn{3}{C}{Shapes3d}\\
	~ & orig. & mod. & noise & orig. & mod. & noise
		\\ \midrule
	{\bf $\beta$-VAE\quad}
	& $0.11\pm0.03$ & $0.08\pm0.11$  & $0.14\pm0.07$  & $0.73\pm0.14$  & $0.43\pm0.06$ & $0.56\pm0.06$ \\ \cmidrule{1-7}
	{\bf Fac. VAE\quad}
	& $0.37\pm0.10$ & $0.27\pm0.11$  & $0.24\pm0.09$  & $0.39\pm0.18$  & $0.25\pm0.08$ & $0.57\pm0.20$ \\ \cmidrule{1-7}
	{\bf TC-$\beta$-VAE\quad}
	& $0.34\pm0.06$ & $0.19\pm0.10$  & $0.27\pm0.03$  & $0.67\pm0.08$  & $0.41\pm0.05$ & $0.59\pm0.09$ \\ \cmidrule{1-7}
	{\bf Slow-VAE\quad}
	& $0.47 \pm 0.07$ & $0.40 \pm 0.07$ & $0.47 \pm 0.08$ &  $0.65 \pm 0.10$ & $0.33 \pm 0.08$ & $0.73 \pm 0.09$\\ \cmidrule{1-7}
	{\bf PCL\quad}
	& $0.28 \pm 0.03$ & $0.30 \pm 0.03$ & $0.29 \pm 0.06$ &  $0.70 \pm 0.06$ & $0.67 \pm 0.09$ & $0.71 \pm 0.07$ \\
	\bottomrule
\end{tabular}
\end{adjustbox}
\label{tbl:dci}
\end{table*}
\begin{table*}[h!] 
\centering
\footnotesize
\caption{FactorVAE Scores for unmodified, modified and noisy datasets. We report the mean and standard deviation over 10 distinct random seeds for each setting. PCL is the only disentangling non-variational model.}\vspace{-.5em}
\begin{adjustbox}{max width=0.65\linewidth}
\begin{tabular}{R@{\hskip1em}*{3}{C}|*{3}{C}<{\clearrow}}
	\toprule
	\setrow{\textbf}
	~ & \multicolumn{3}{C}{dSprites} & \multicolumn{3}{C}{Shapes3d}\\
	~ & orig. & mod. & noise & orig. & mod. & noise
		\\ \midrule
	{\bf $\beta$-VAE\quad}
	& $0.47\pm0.07$ & $0.38\pm0.13$  & $0.50\pm0.10$  & $0.80 \pm 0.17$  & $0.54 \pm 0.10$ & $0.71\pm0.06$ \\ \cmidrule{1-7}
	{\bf Fac. VAE\quad}
	& $0.67\pm0.11$ & $0.62\pm0.14$  & $0.60\pm0.11$  & $0.63\pm 0.15$  & $0.48 \pm 0.05$ & $0.71\pm0.15$ \\ \cmidrule{1-7}
	{\bf TC-$\beta$-VAE\quad}
	& $0.68\pm0.09$ & $0.53\pm0.15$  & $0.60\pm0.12$  & $0.76\pm0.07$  & $0.57\pm0.07$ & $0.71\pm0.06$ \\ \cmidrule{1-7}
	{\bf Slow-VAE\quad}
	& $0.77 \pm 0.03$ & $0.77 \pm 0.04$ & $0.76 \pm 0.07$ &  $0.87 \pm 0.10$ & $0.62 \pm 0.06$ & $0.85 \pm 0.08$ \\ \cmidrule{1-7}
	{\bf PCL\quad}
	& $0.77 \pm 0.09$ & $0.82 \pm 0.05$ & $0.77 \pm 0.08$ & $0.80 \pm 0.06$ & $0.77 \pm 0.07$ & $0.80 \pm 0.06$ \\
	\bottomrule
\end{tabular}
\end{adjustbox}
\label{tbl:factorscore}
\end{table*}
\begin{table*}[h!] 
\centering
\footnotesize
\caption{SAP Scores for unmodified, modified and noisy datasets. We report the mean and standard deviation over 10 distinct random seeds for each setting. PCL is the only disentangling non-variational model.}\vspace{-.5em}
\begin{adjustbox}{max width=0.65\linewidth}
\begin{tabular}{R@{\hskip1em}*{3}{C}|*{3}{C}<{\clearrow}}
	\toprule
	\setrow{\textbf}
	~ & \multicolumn{3}{C}{dSprites} & \multicolumn{3}{C}{Shapes3d}\\
	~ & orig. & mod. & noise & orig. & mod. & noise
		\\ \midrule
	{\bf $\beta$-VAE\quad}
	& $0.04\pm0.01$ & $0.02\pm0.02$  & $0.03\pm0.03$  & $0.16\pm0.08$  & $0.03\pm0.03$ & $0.09\pm0.02$ \\ \cmidrule{1-7}
	{\bf Fac. VAE\quad}
	& $0.07\pm0.03$ & $0.06\pm0.03$  & $0.08\pm0.01$  & $0.07\pm0.04$  & $0.04\pm0.01$ & $0.08\pm0.03$ \\ \cmidrule{1-7}
	{\bf TC-$\beta$-VAE\quad}
	& $0.08\pm0.01$ & $0.06\pm0.03$  & $0.05\pm0.02$  & $0.08\pm0.02$  & $0.04\pm0.02$ & $0.06\pm0.03$ \\ \cmidrule{1-7}
	{\bf Slow-VAE\quad}
	& $0.08 \pm 0.01$ & $0.07 \pm 0.01$ & $0.07 \pm 0.01$ &  $0.09 \pm 0.04$ & $0.04 \pm 0.01$ & $0.09 \pm 0.05$ \\ \cmidrule{1-7}
	{\bf PCL\quad}
	& $0.07 \pm 0.03$ & $0.10 \pm 0.03$ & $0.10 \pm 0.03$ & $0.07 \pm 0.01$ & $0.07 \pm 0.01$ & $0.07 \pm 0.01$ \\
	\bottomrule
\end{tabular}
\end{adjustbox}
\label{tbl:sapscore}
\end{table*}
\begin{figure}[h!]
\begin{center}
\begin{minipage}[t]{0.4\textwidth}
\includegraphics[width=\linewidth]{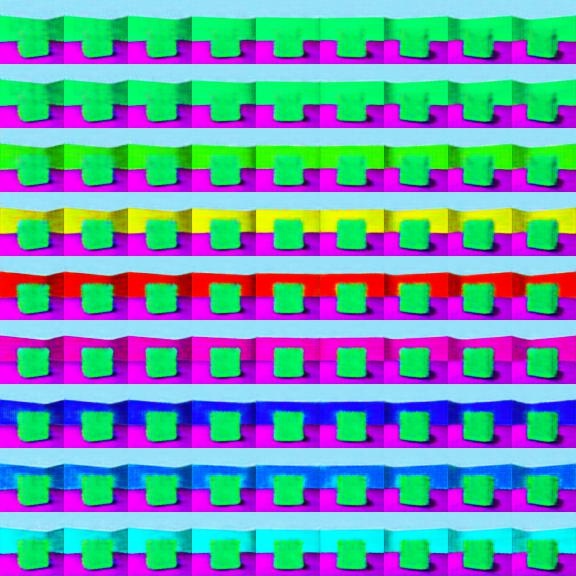}
\end{minipage}
\hspace{0.5cm}
\begin{minipage}[t]{0.4\textwidth}
\includegraphics[width=\linewidth]{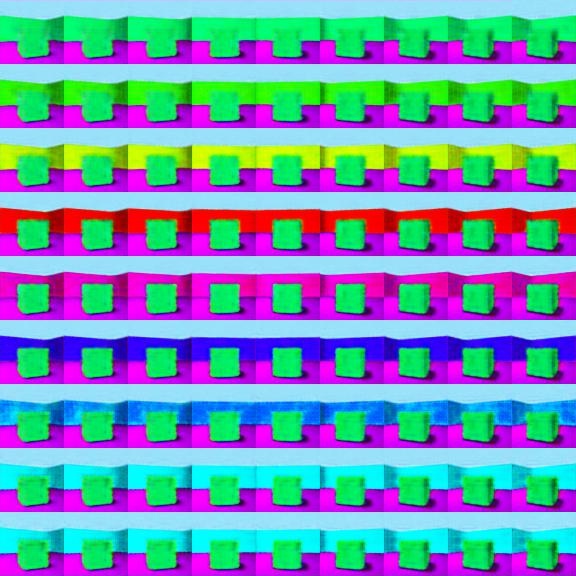}
\end{minipage}\\
\vspace*{0.5cm}
\begin{minipage}[t]{0.4\textwidth}
\includegraphics[width=\linewidth]{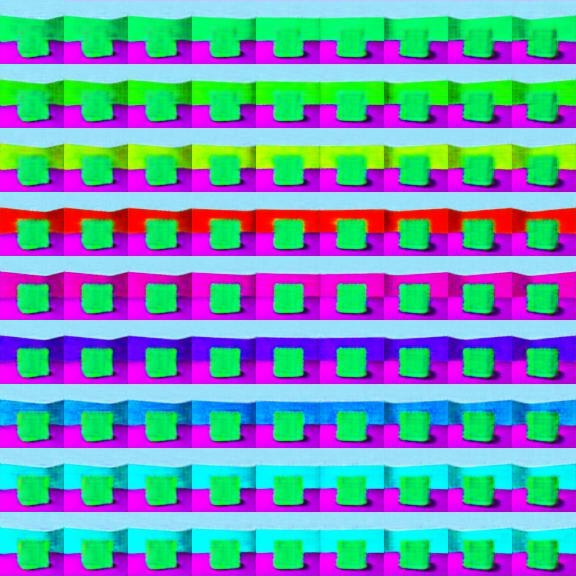}
\end{minipage}
\hspace{0.5cm}
\begin{minipage}[t]{0.4\textwidth}
\includegraphics[width=\linewidth]{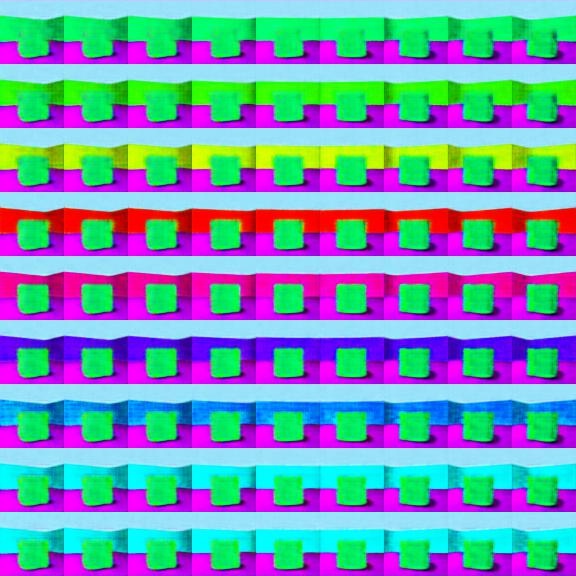}
\end{minipage}
\end{center}
    \caption{Latent traversals along two latent dimensions for four different disentangled representations. They encode the wall hue and orientation separately. The latent coordinates were flipped to match the same alignment.}
    \label{fig:cartesian}
\end{figure}

\begin{figure}[h!]
\begin{center}
\begin{minipage}[t]{0.4\textwidth}
\includegraphics[width=\linewidth]{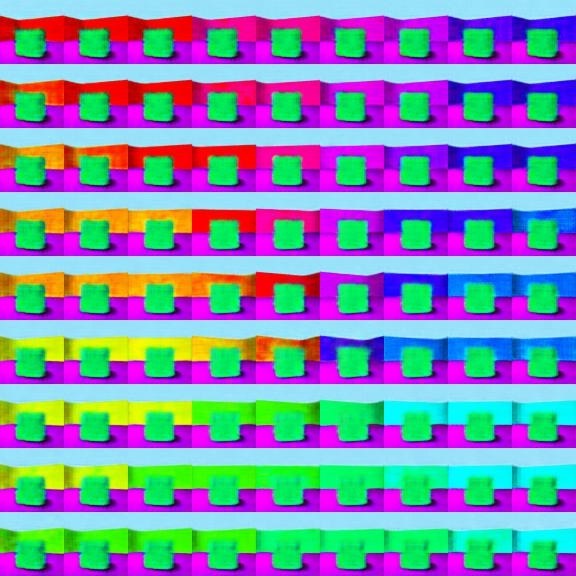}
\end{minipage}
\hspace{0.5cm}
\begin{minipage}[t]{0.4\textwidth}
\includegraphics[width=\linewidth]{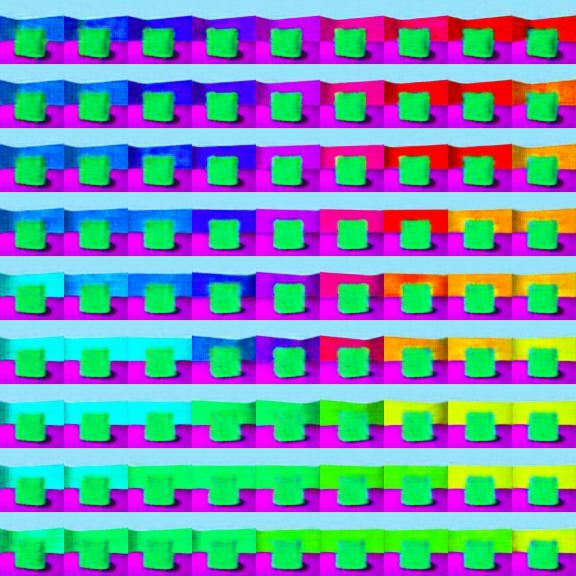}
\end{minipage}\\
\vspace*{0.5cm}
\begin{minipage}[t]{0.4\textwidth}
\includegraphics[width=\linewidth]{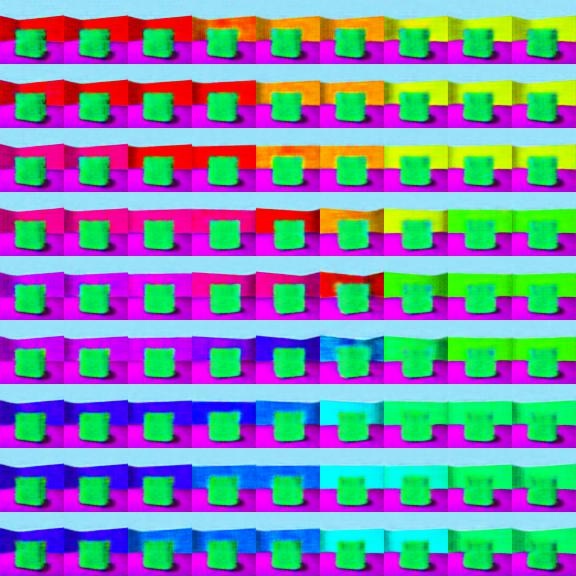}
\end{minipage}
\hspace{0.5cm}
\begin{minipage}[t]{0.4\textwidth}
\includegraphics[width=\linewidth]{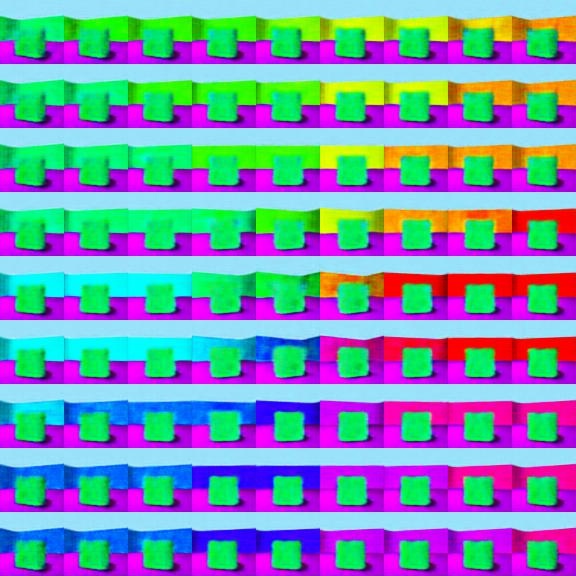}
\end{minipage}
\end{center}
    \caption{Latent traversals along two latent dimensions for four different disentangled representations. They encode a mixture of wall hue and orientation.}
    \label{fig:polar}
\end{figure}

\end{document}

%% file: procedure_visualisation.tex
\begin{figure*}[h!]
    \begin{center}
    \begin{adjustbox}{max width=\linewidth}
      \begin{tikzpicture}[
        thick, text centered,
        box/.style={draw, minimum width=0.7cm, minimum height=0.7cm},
        box_image/.style={draw, minimum width=1.1cm, minimum height=1.1cm},
        func/.style={circle, text=white},
      ]

    	\node (i) at (0,0) {\includegraphics[trim=50px 20px 50px 20px,clip, width=0.24\textwidth]{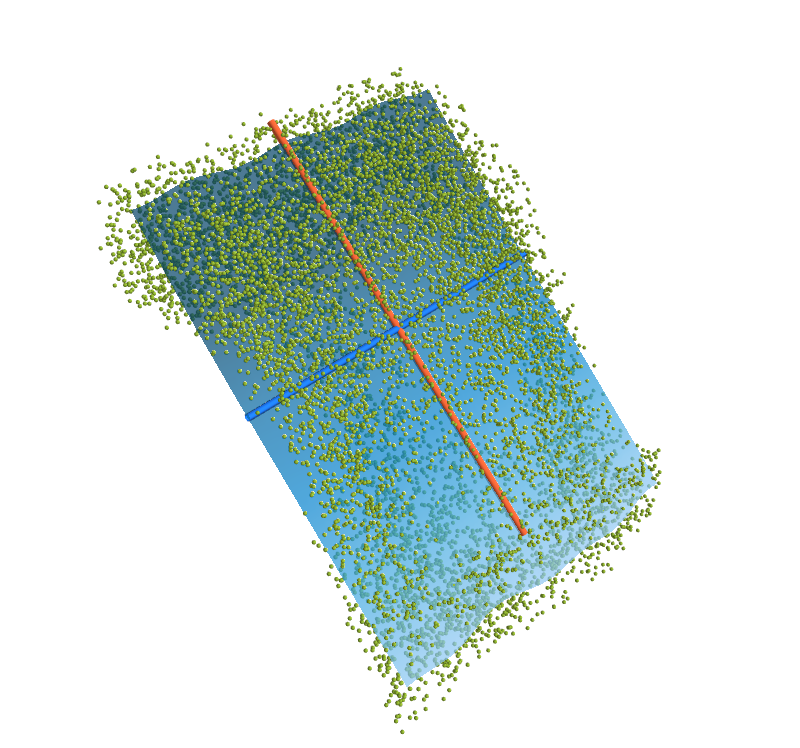}};
    	\node(itxt) [below left=0.3cm and -0.3cm of i.north west] {(i)};
    	
    	\node (ii) [right=0.01cm of i] {\includegraphics[trim=50px 20px 50px 20px,clip, width=0.24\textwidth]{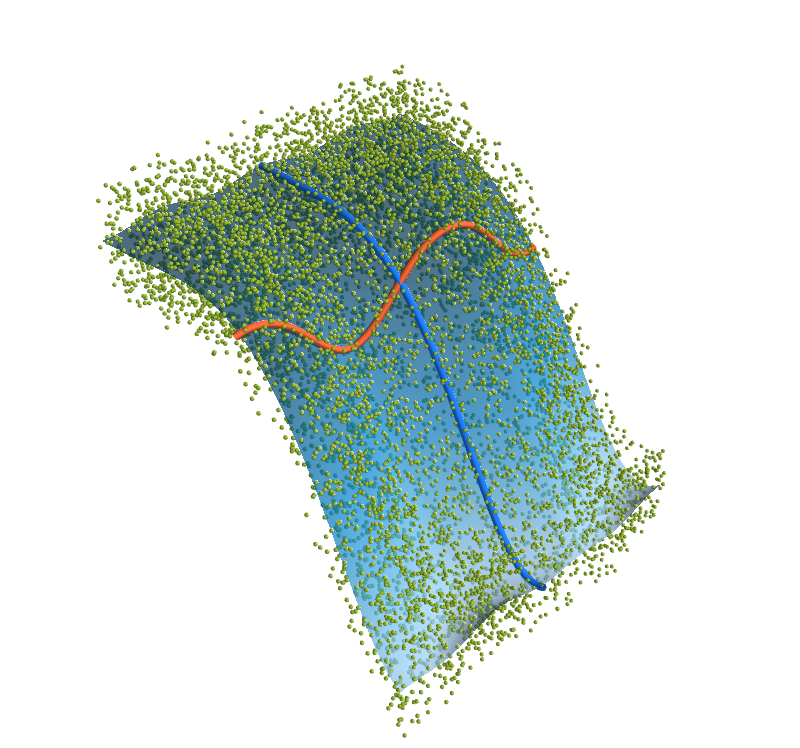}};
    	\node(iitxt) [below left=0.3cm and -0.3cm of ii.north west] {(ii)};
    	
    	\node (iii) [right=0.01cm of ii] {\includegraphics[trim=50px 20px 50px 20px,clip, width=0.24\textwidth]{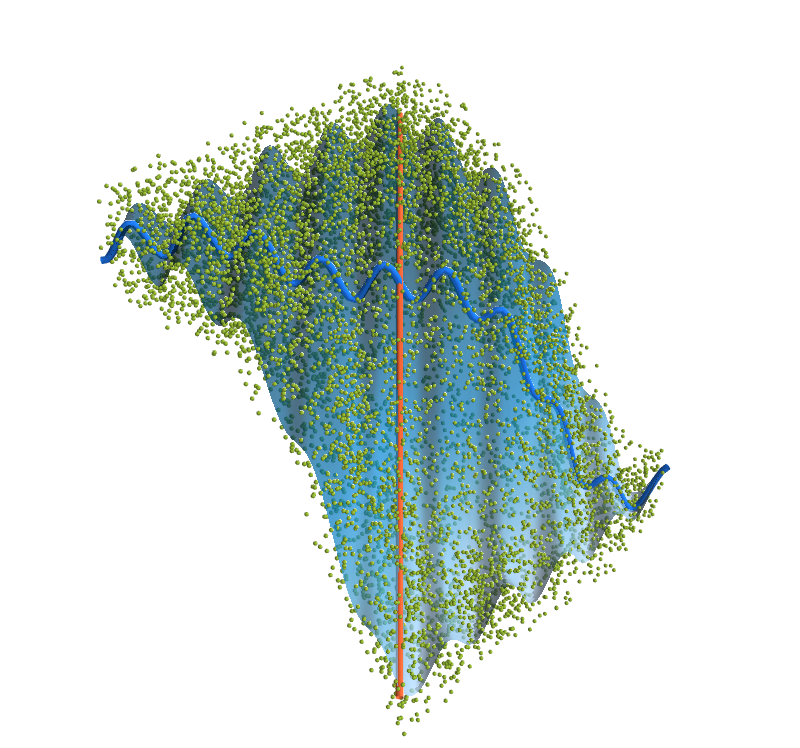}};
    	\node(iiitxt) [below left=0.3cm and -0.3cm of iii.north west] {(iii)};

    	\node (iv) [right=0.01cm of iii] {\includegraphics[trim=50px 20px 50px 20px,clip, width=0.24\textwidth]{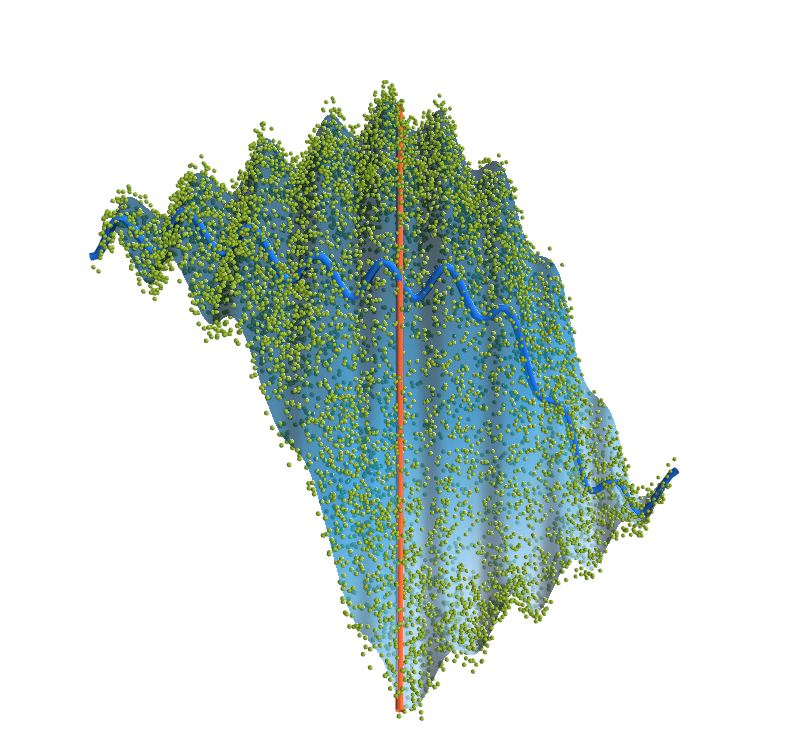}};
        \node(ivtxt) [below left=0.3cm and -0.3cm of iv.north west] {(iv)}; 
    	    	\node (arrow) [right=-0.4cm of iii] {$\rightarrow$};

    \end{tikzpicture}	
    \end{adjustbox}
    \end{center}
        \caption{
        Illustrations for linear and nonlinear embeddings. From left to right: (i) a 3 dimensional point cloud and the corresponding two-dimensional PCA manifold (blue surface) with the canonical principal components (red/blue curves), (ii) a nonlinear two-dimensional manifold with a latent traversal, (iii) a locally perturbed two-dimensional manifold with its principal components which are rotated with respect to (ii), (iv) the goal of our modifications is to move each datapoint closer to this \textit{entangled} manifold.
    }    
    \label{fig:visualization_modification}
\end{figure*}

%% file: architecture_figure.tex
\tikzset
{
  decoder/.pic =
  {
    \draw [fill=#1] (0,0) -- (0,\b) -- (\a,\c) -- (\a,-\c) -- (0,-\b) -- cycle ;
    \coordinate (-center) at (\a/2,0);
    \coordinate (-out) at (\a,0);
    \coordinate (-west) at (\a,0);
    \coordinate (-northwest) at (\a,-\c);
    \coordinate (-southwest) at (\a,\c);
    \coordinate (-northeast) at (0,-\b);
    \coordinate (-southeast) at (0,\b);

    \coordinate (-east) at (0,0);
  },
  myArrows/.style=
  {
    line width=2mm,
    red,
    -{Triangle[length=1.5mm,width=5mm]},
    shorten >=2pt,
    shorten <=2pt,
  }
}
    \def\a{2.}  
    \def\b{.45}
    \def\c{1}

\begin{figure*}[h!]
\begin{center}
  \begin{tikzpicture}[
    thick, text centered,
    box/.style={draw, minimum width=0.7cm, minimum height=0.7cm},
    box_image/.style={draw, minimum width=1.3cm, minimum height=1.3cm},
    func/.style={circle, text=white},
  ]

  \node[box, fill=gray!20, align=center] (w1) {$w_1$};
  \node[box, fill=gray!20, align=center] (w2) [below=0.0cm of w1]{$w_2$};
  \node[box, fill=gray!20, align=center] (wdot) [below=0.0cm of w2]{$w_3$};
  \node[box, fill=gray!20, align=center] (wnm1) [below=0.0cm of wdot]{$w_4$};
  \node[box, fill=gray!20, align=center] (wn) [below=0.0cm of wnm1]{$w_5$};
  \node[box_image, fill=white!20] (x) [right=1cm of w1] {$\mb{x}$};

  \node[box_image, fill=white!20] (m) [below=1cm of x] {$m_\psi(\mb{w})$};

  \node[box_image, fill=white!20] (xp) [below=0.65cm of m] {$\mb{x}'$};

  \pic[] (e) [below right=-0.2cm and 3cm of m, rotate=180] {decoder={our-lightblue}} ;
  \pic[] (d) [right=3.5cm of e-center] {decoder={our-red}} ;
  \node at (e-center) {$e_{\varphi_\mathrm{dis}}$} ;
  \node at (d-center) {$d_{\theta_\mathrm{dis}}$};

  \pic[] (ep) [right=3cm of x, rotate=180] {decoder={our-lightblue}} ;
  \pic[] (dp) [right=3.5cm of ep-center] {decoder={our-green}} ;
  \node at (ep-center) {$e_{\varphi_\mathrm{ent}}$} ;
  \node at (dp-center) {$d_{\theta_\mathrm{ent}}$};

    \node[box_image, fill=white!20] (xtildedis) [right= 2cm of d-center] {$\Tilde{\mb{x}}_\mathrm{dis}$};
    \node[box_image, fill=white!20] (xtildeent) [right=2cm of dp-center] {$\Tilde{\mb{x}}_\mathrm{ent}$};

    \node[] (lm) [below right=-0.5cm and 0.5cm of xp] {
    $\begin{aligned}
        \psi^\star &= \argmin_\psi
        \mathcal{L}_m\\
        \mathcal{L}_m &=
        \mathcal{L}_\mathrm{ent} - \mathcal{L}_\mathrm{dis}
    \end{aligned}$
    };

    \node[] (lddis) [below right=0.7cm and 1.8cm of d-center] {
    $\begin{aligned}
        \theta^\star_\mathrm{dis} =& \argmin_{\theta_\mathrm{dis}} \mathcal{L}_\mathrm{dis}\\
        \mathcal{L}_\mathrm{dis} =& \norm{\Tilde{\mb{x}}_\mathrm{dis} - \mb{x}'}^2
    \end{aligned}$
    };

    \node[] (ldent) [above right = 0.7cm and 1.8cm of dp-center] {
    $\begin{aligned}
        \theta^\star_\mathrm{ent} =& \argmin_{\theta_\mathrm{ent}} \mathcal{L}_\mathrm{ent}\\
        \mathcal{L}_\mathrm{ent} =& \norm{\Tilde{\mb{x}}_\mathrm{ent} - \mb{x}'}^2
    \end{aligned}$
    };

    \node [] (xp_ldent1) [below right = 0cm and 2.5cm of ldent.west]{};
    \node [] (xp_ldent2) [below right = -.3cm and 2.5cm of ldent.west]{};
    \node [] (xp_lddis1) [below right = 0cm and 2.5cm of lddis.west]{};

  \fill[our-orange, opacity=0.6] (w1.north east) -- (x.north west) -- (x.south west) -- (wn.south east);

  \fill[our-orange, opacity=0.6] (x.north east) -- (e-northwest) -- (e-southwest) -- (x.south east);
  \fill[our-orange, opacity=0.6] (x.north east) -- (ep-northwest) -- (ep-southwest) -- (x.south east);

  \fill[our-orange, opacity=0.6] (e-northeast) -- (d-southeast) -- (d-northeast) -- (e-southeast);
  \fill[our-orange, opacity=0.6] (ep-northeast) -- (dp-southeast) -- (dp-northeast) -- (ep-southeast);

  \fill[our-orange, opacity=0.6] (d-southwest) -- (xtildedis.north west) -- (xtildedis.south west) -- (d-northwest);

  \fill[our-orange, opacity=0.6] (dp-southwest) -- (xtildeent.north west) -- (xtildeent.south west) -- (dp-northwest);

  \node[align=center] (txt1) [right=0.15cm of e-east] {
  $\mb{z}_\mathrm{dis} \sim$ \\
  $\mathcal{N}(\bm{\mu}_\mathrm{dis}, \bm{\sigma}^2_\mathrm{dis})$
  };
\node[align=center] (txt) [right=0.15cm of ep-east]{
  $\mb{z}_\mathrm{ent} \sim$ \\
  $\mathcal{N}(\bm{\mu}_\mathrm{ent}, \bm{\sigma}^2_\mathrm{ent})$
  };

  \fill[our-lightblue, opacity=0.4] (w1.north east) -- (m.north west) -- (m.south west) -- (wn.south east);

  \fill[our-orange, opacity=0.6] (x.south east) -- (m.north east) -- (m.north west) -- (x.south west);

  \fill[our-lightblue, opacity=0.4] (m.south east) -- (xp.north east) -- (xp.north west) -- (m.south west);

  \node[align=center] (pluseps) [below=0.1cm of x] {$+$\\$\epsilon$};
  \node[align=center] (equals) [below=0.075cm of m] {$\veq$};

\end{tikzpicture}

	\caption{A schematic visualization of the image generation process. Starting from ground truth generating factors $\mb{w}$, two $\beta$-VAE encoder-decoder pairs are initialized such that one (top) produces entangled and the other (bottom) disentangled representations. Another decoder-like network $m$ is trained to produce additive manipulations to the original images $x$. The encoder of the entangling model is frozen and fed with the original images. The set of ground truth generating factors $\mb{w}$ stays untouched by the modification.}
	\label{fig:modification_architecture}
\end{center}
\end{figure*}

%% file: ds_figure.tex
\begin{figure}[h!]
\begin{center}
\begin{adjustbox}{max width=.8\linewidth}
  \begin{tikzpicture}[
    thick, text centered,
    box/.style={draw, minimum width=0.7cm, minimum height=0.7cm},
    box_image/.style={draw, minimum width=1.1cm, minimum height=1.1cm},
    func/.style={circle, text=white},
  ]

	\node (x00) at (0,0) {\includegraphics{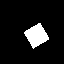}};
	\node (x00p) [right=0.01cm of x00] {$+\  \varepsilon$};
	\node (x01) [right=0.01cm of x00p] {\includegraphics{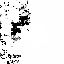}};
	\node (x01e) [right=0.01cm of x01] {$=$};
	\node (x02) [right=0.05cm of x01e] {\includegraphics{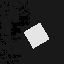}};
	
	\node (x10) [below=0.05cm of x00] {\includegraphics{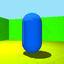}};
	\node (x10p) [right=0.01cm of x10] {$+\  \varepsilon$};
	\node (x11) [right=0.05cm of x10p] {\includegraphics{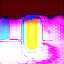}};
	\node (x11e) [right=0.01cm of x11] {$=$};
	\node (x12) [right=0.05cm of x11e] {\includegraphics{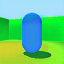}};
	
\end{tikzpicture}	
\end{adjustbox}
\end{center}
	\caption{From left to right: Original images, additive manipulations and the altered images. Top row shows an example of dSprites, the bottom for Shapes3D.}
\label{fig:example_manipulations}
\end{figure}

%% file: proof_vis.tex
    \begin{center}
    \begin{adjustbox}{max width=\linewidth}
      \begin{tikzpicture}[
        thick, text centered,
        box/.style={draw, minimum width=0.7cm, minimum height=0.7cm},
        box_image/.style={draw, minimum width=1.1cm, minimum height=1.1cm},
        func/.style={circle, text=white},
      ]

    	\node (img) at (0,0) {\includegraphics[width=\textwidth]{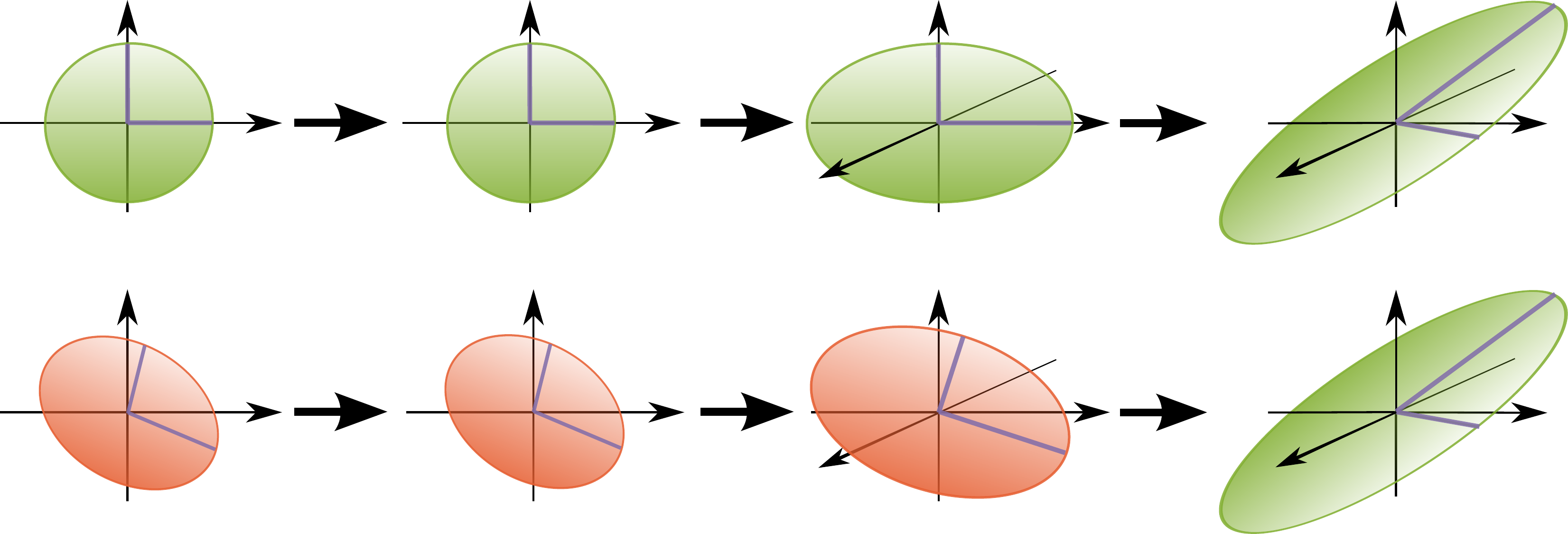}};
    	\node(z) [below right=-0.5cm and 1.25cm of img.north west] {$Z$};
    	\node(ztz) [below right=2.7cm and 0.3cm of img.north west] {$\mathrm{diag}\left(Z^\top Z\right) = \bm{1}$};
        
        \node(x) [below right=-0.5cm and 15.15cm of img.north west] {$\hat{X}$};

        \node(v) [below right=0.7cm and 3.5cm of img.north west] {$V^\top$};
        \node(sigma) [right=3.6cm of v] {$\Sigma$};
        \node(u) [right=4.15cm of sigma] {$U$};
        
        \node(v2) [below = 2.65cm of v] {$V^\top$};
        \node(sigma2) [below = 2.65cm of sigma] {$\Sigma$};
        \node(u2) [below = 2.65cm of u] {$U$};
    \end{tikzpicture}	
    \end{adjustbox}
    \end{center}
        \captionof{figure}{
        The SVD decomposition of a VAE decoder (top) and an alternative decoder (bottom) which decodes the same data $\hat{X}$, complies with $V=\mathcal{I}$, and also shares $\mathrm{diag}\left(Z^\top Z \right) = \bm{1}$. The difference lies in the rotation induced by $U$, which for VAEs (and PCA) aligns the directions of largest variance in $\hat{X}$ with the cartesian axes.
    }    
    \label{fig:svd_proof}